\documentclass[a4paper,11pt]{article}
\pdfoutput=1
\usepackage{a4wide}
\usepackage{amsmath,amsfonts,amssymb,amsthm}
\usepackage[colorlinks,citecolor=blue]{hyperref}
\usepackage{enumitem}
\usepackage{bm}
\usepackage{natbib}
\setcitestyle{square}

\newtheorem{theorem}{Theorem}[section]
\newtheorem{lemma}[theorem]{Lemma}
\newtheorem{proposition}[theorem]{Proposition}
\newtheorem{corollary}[theorem]{Corollary}

\theoremstyle{definition}
\newtheorem{definition}[theorem]{Definition}

\newtheorem{remark}[theorem]{Remark}

\DeclareMathOperator*{\argmin}{argmin}

\numberwithin{equation}{section}

\def \bE {\mathbb{E}}

\def \bN {\mathbb{N}}

\def \bR {\mathbb{R}}

\def \cA {\mathcal{A}}
\def \cB {\mathcal{B}}

\def \cE {\mathcal{E}}
\def \cF {\mathcal{F}}
\def \cG {\mathcal{G}}
\def \cH {\mathcal{H}}
\def \cI {\mathcal{I}}

\def \cL {\mathcal{L}}
\def \cM {\mathcal{M}}
\def \cN {\mathcal{N}}
\def \cO {\mathcal{O}}

\def \cR {\mathcal{R}}
\def \cS {\mathcal{S}}

\def \cU {\mathcal{U}}

\def \cW {\mathcal{W}}

\def \Ba {{\boldsymbol{a}}}
\def \Bb {{\boldsymbol{b}}}

\def \Bi {{\boldsymbol{i}}}

\def \Bn {{\boldsymbol{n}}}

\def \Bs {{\boldsymbol{s}}}

\def \Bx {{\boldsymbol{x}}}
\def \By {{\boldsymbol{y}}}

\def \Lip {\,{\rm Lip}\,}
\def \Id {\,{\rm Id}\,}

\begin{document}
\title{Approximation bounds for norm constrained neural networks with applications to regression and GANs}
\author{
Yuling Jiao \thanks{School of Mathematics and Statistics, and
Hubei Key Laboratory of Computational Science, Wuhan University, Wuhan 430072, China. (yulingjiaomath@whu.edu.cn)}
\and
Yang Wang \thanks{Department of Mathematics, The Hong Kong University of Science and Technology, Clear Water Bay, Kowloon, Hong Kong, China. (yangwang@ust.hk)}
\and
Yunfei Yang \thanks{Department of Mathematics, The Hong Kong University of Science and Technology, Clear Water Bay, Kowloon, Hong Kong, China. (Corresponding author, yyangdc@connect.ust.hk)}
}
\date{}
\maketitle

\begin{abstract}
This paper studies the approximation capacity of ReLU neural networks with norm constraint on the weights. We prove upper and lower bounds on the approximation error of these networks for smooth function classes. The lower bound is derived through the Rademacher complexity of neural networks, which may be of independent interest. We apply these approximation bounds to analyze the convergences of regression using norm constrained neural networks and distribution estimation by GANs. In particular, we obtain convergence rates for over-parameterized neural networks. It is also shown that GANs can achieve optimal rate of learning probability distributions, when the discriminator is a properly chosen norm constrained neural network.


\end{abstract}

\section{Introduction}

The expressiveness and approximation capacity of neural networks has been an active research area in the past few decades. The universal approximation property of shallow neural networks with one hidden layer and various activation functions was widely discussed in the 1990s \citep{cybenko1989approximation,hornik1991approximation,pinkus1999approximation}. It was also shown that shallow neural networks can achieve attractive approximation rates for certain functions \citep{barron1993universal}. The recent breakthrough of deep learning has attracted much research on the approximation theory of deep neural networks. The approximation rates of ReLU deep neural networks have been well studied for many function classes, such as continuous functions \citep{yarotsky2017error,yarotsky2018optimal,shen2020deep}, smooth functions \citep{yarotsky2020phase,lu2021deep}, piecewise smooth functions \citep{petersen2018optimal}, shift-invariant spaces \citep{yang2022approximation} and band-limited functions \citep{montanelli2021deep}.

In practice, neural network models are trained by minimizing certain loss functions on observed data. The approximation theory provides estimates on the bias of the model, while the sample complexity of the model controls how well it can generalize to unseen data by learning from finite observed samples \citep{anthony2009neural,shalevshwartz2014understanding,mohri2018foundations}. In modern applications, the number of training samples is often smaller than the number of weights in neural networks. For the generalization performance in this case, as pointed out by \citet{bartlett1998sample}, the size of the weights is more important than the size of networks. The recent works \citep{neyshabur2015norm,bartlett2017spectrally,golowich2020size,barron2019complexity} also show that the sample complexity of deep neural networks can be controlled by certain norms of the weights. However, in the approximation theory literature, the approximation rates of deep neural networks are characterized by the number of weights \citep{yarotsky2017error,yarotsky2018optimal,yarotsky2020phase} or the number of neurons \citep{shen2020deep,lu2021deep}, rather than the size of weights.

Besides, many regularization methods have been introduced to enforce Lipschitz constraint on neural networks (for example, spectral normalization \citep{miyato2018spectral} and weight penalty \citep{brock2019large}). It has been demonstrated that the Lipschitz constraint on neural networks can improve robustness to adversarial examples \citep{cisse2017parseval}, and stabilize the training of Generative Adversarial Networks (GAN, \citep{goodfellow2014generative,arjovsky2017towards,arjovsky2017wasserstein}). However, these regularization methods often make explicit or implicit restrictions on some norms of the weights, which largely reduce the expressive power of the models. For instance, \citet{huster2019limitations} showed that ReLU neural networks with certain constraints on the weights cannot represent some simple functions, such as the absolute value function. Hence, it is desirable to study how norm constrains on the weights affect the approximation capacity of neural networks.

In this paper, we give upper and lower bounds on the approximation error of ReLU neural networks with certain norm constrain on the weights for smooth function classes. To be concrete, let $\phi_\theta:\bR^d \to \bR$ be a function computed by a multi-layer ReLU neural network with width $W$ and depth $L$, where $\theta$ represents the collection of weights. In the $\ell$-th layer, the neural network computes an affine transformation $T_\ell(\Bx) = A_\ell \Bx +\Bb_\ell$ and then applies the ReLU activation function element-wise (no activation in the output layer). When all the biases $\Bb_\ell =\boldsymbol{0}$, it is natural to consider the constraint on the product of matrix norm $\prod_{\ell =0}^L \|A_\ell\| \le K$, which controls the generalization ability \citep{bartlett2017spectrally,golowich2020size}. We generalize this idea to general bias $\Bb_\ell$ and define the norm constraint $\kappa(\theta) \le K$ as (\ref{norm constraint}), which suitably constrains the bias. Our main results estimate the approximation error for H\"older continuous function $f\in \cH^\alpha(\bR^d)$ with smoothness index $\alpha>0$. We show that if the width $W$ and depth $L$ are sufficiently large ($W$ needs to grow with $K$), then it holds that
\[
\sup_{f\in \cH^\alpha} \inf_{\kappa(\theta)\le K} \| f - \phi_\theta\|_{C([0,1]^d)} \lesssim K^{-\alpha/(d+1)}.
\]
In addition, if $d>2\alpha$, then for any neural networks with width $W\ge 2$ and depth $L$,
\[
\sup_{f\in \cH^\alpha} \inf_{\kappa(\theta)\le K} \| f - \phi_\theta\|_{C([0,1]^d)} \gtrsim (K \sqrt{L})^{-2\alpha/(d-2\alpha)}.
\]

The advantage of our approximation upper bound is that it only depends on the norm constraint so that it can be combined with the generalization bounds in \citep{bartlett2017spectrally,golowich2020size} and applied to over-parameterized neural networks. For comparison, in \citep{shen2020deep,lu2021deep}, the approximation error is bounded by the width and depth, but there is no restriction on the weights. \citet{yarotsky2017error,yarotsky2018optimal,yarotsky2020phase} obtained approximation bounds in terms of the number of non-zero weights. Although this can be regarded as the sum of zero-norm of the weights, it is more like a constraint on the network architecture, rather than a constraint on the size of the weights. In \citep{petersen2018optimal,boelcskei2019optimal,schmidthieber2021kolmogorov}, the authors also provide approximation results of deep neural networks  with bound  on the maximum value of the weights. But these bounds can not directly control the generalization. On the contrary, our norm constraint provides a bound on the Rademacher complexity of the network (see Lemma \ref{Rademacher bound} and \citet{golowich2020size}).

To illustrate the application of the approximation bounds, we study the regression problem of estimating an unknown function $f_0\in \cH^\alpha$ from its noisy samples. Combining the empirical process theory with our approximation bounds, we can estimate the convergence rate of the empirical risk minimization using norm constrained neural networks. In particular, we obtain convergence rates for over-parameterized neural networks, which give statistical guarantee for neural networks used in practice. We also apply our results to generative adversarial networks. It is shown that, if a properly chosen norm constrained neural network is used as the discriminator, GAN is able to achieve the optimal convergence rate of learning probability distributions.

The rest of the paper is organized as follows. In Section \ref{sec: NN}, we define the norm constraint on neural networks and give some preliminary results. Section \ref{sec: approximation} presents and proves our main results on the approximation bounds for norm constrained neural networks. In Section \ref{sec: application}, we apply our results to study the convergence rates of two machine learning algorithms. Finally, Section \ref{sec: conclusion} concludes this paper with a discussion on possible future directions of research.

\subsection{Notation}

The set of positive integers is denoted by $\bN:=\{1,2,\dots\}$. For convenience, we also use the notation $\bN_0:= \bN \cup \{0\}$.
The cardinality of a set $S$ is denoted by $|S|$.
We use $\|\Bx\|_p$ to denote the $p$-norm of a vector $\Bx\in\bR^d$.
For a multi-index $\Bs=(s_1,\dots,s_d)\in \bN_0^d$, the symbol $\partial^\Bs$ denotes the partial differential operator $\partial^\Bs:= (\frac{\partial}{\partial x_1})^{s_1} \dots (\frac{\partial}{\partial x_d})^{s_d}$ and we use the convention that $\partial^\Bs$ is the identity operator when $\Bs=\boldsymbol{0}$.
If $X$ and $Y$ are two quantities, we denote $X\land Y:= \min\{X,Y\}$ and $X\lor Y:= \max\{X,Y\}$. We use $X \lesssim Y$ or $Y \gtrsim X$ to denote the statement that $X\le CY$ for some constant $C>0$. We denote $X \asymp Y$ when $X \lesssim Y \lesssim X$.
Finally, we introduce the covering number and packing number to measure the complexity of a set in a metric space.

\begin{definition}[Covering and Packing numbers]
Let $\rho$ be a metric on $\cM$ and $S\subseteq \cM$. For $\epsilon>0$, a set $T \subseteq \cM$ is called an $\epsilon$-covering (or $\epsilon$-net) of $S$ if for any $x\in S$ there exists $y\in T$ such that $\rho(x,y)\le \epsilon$. A subset $U \subseteq S$ is called an $\epsilon$-packing of $S$ (or $\epsilon$-separated) if any two elements $x\neq y$ in $U$ satisfy $\rho(x,y)>\epsilon$. The $\epsilon$-covering and $\epsilon$-packing numbers of $S$ are denoted respectively by
\begin{align*}
\cN_c(S,\rho,\epsilon) &:= \min\{|T|: T \mbox{ is an $\epsilon$-covering of } S \}, \\
\cN_p(S,\rho,\epsilon) &:= \max\{|U|: U \mbox{ is an $\epsilon$-packing of } S \}.
\end{align*}
\end{definition}
It is not hard to check that $\cN_p(S,\rho,2\epsilon) \le \cN_c(S,\rho,\epsilon) \le \cN_p(S,\rho,\epsilon)$.

\section{Neural networks with norm constraints}\label{sec: NN}

Let $L,N_1,\dots, N_L \in \bN$. We consider the function $\phi:\bR^{d} \to \bR^{k}$ that can be parameterized by a ReLU neural network of the form
\begin{equation}\label{NN standard form}
\begin{aligned}
\phi_0(\Bx) &= \Bx, \\
\phi_{\ell+1}(\Bx) &= \sigma(A_{\ell} \phi_{\ell}(\Bx)+\Bb_\ell), \quad \ell = 0,\dots,L-1, \\
\phi(\Bx) &= A_L \phi_L(\Bx),
\end{aligned}
\end{equation}
where $A_\ell \in \bR^{N_{\ell+1}\times N_{\ell}}$, $\Bb_\ell\in \bR^{N_{\ell+1}}$ with $N_0 =d$ and $N_{L+1} =k$. The activation function $\sigma(x) := x\lor 0$ is the Rectified Linear Unit function (ReLU, \citep{nair2010rectified}) and it is applied element-wise. The numbers $W:=\max\{N_1,\dots,N_L\}$ and $L$ are called the width and depth of the neural network, respectively. We denote by $\cN\cN_{d,k}(W,L)$ the set of functions that can be parameterized by ReLU neural networks with width $W$ and depth $L$. When the input dimension $d$ and output dimension $k$ are clear from contexts, we simply denote it by $\cN\cN(W,L)$. Sometimes, we will use the notation $\phi_\theta \in \cN\cN(W,L)$ to emphasize that the neural network function $\phi_\theta$ is parameterized by
\[
\theta:= ((A_0,\Bb_0),\dots,(A_{L-1},\Bb_{L-1}), A_L).
\]

Next, we introduce a special class of neural network functions $\cS\cN\cN(W,L)$ which contains functions of the form
\begin{equation}\label{NN special form}
\tilde{\phi}(\Bx) = \tilde{A}_L \sigma(\tilde{A}_{L-1}\sigma(\cdots \sigma(\tilde{A}_0 \tilde{\Bx}) ) ), \quad \tilde{\Bx} :=
\begin{pmatrix}
\Bx \\
1
\end{pmatrix},
\end{equation}
where $\tilde{A}_\ell \in \bR^{N_{\ell+1} \times N_\ell}$ with $N_0=d+1$ and $\max\{N_1,\dots,N_L\}= W$. Since these functions can also be written in the form (\ref{NN standard form}) with $\Bb_\ell =\boldsymbol{0}$ for all $1\le \ell \le L-1$, we know that $\cS\cN\cN(W,L)\subseteq \cN\cN(W,L)$. There is a natural way to introduce norm constraint on the weights: for any $K\ge 0$, we denote by $\cS\cN\cN(W,L,K)$ the set of functions in the form (\ref{NN special form}) that satisfies
\[
\prod_{\ell=0}^L \|\tilde{A}_\ell\| \le K,
\]
where $\|A\|$ is some norm of a matrix $A = (a_{i,j}) \in \bR^{m\times n}$ and, for simplicity, we only consider the operator norm defined by $\| A\| := \sup_{\|\Bx\|_\infty \le 1} \|A\Bx\|_\infty$ in this paper. It is well-known that $\| A\|$ is the maximum $1$-norm of the rows of $A$:
\[
\| A\| = \max_{1\le i\le m} \sum_{j=1}^{n} |a_{i,j}|.
\]
Hence, we make a constraint on the $1$-norm of the incoming weights of each neuron.

To introduce norm constraint for the class $\cN\cN(W,L)$, we observe that any $\phi \in \cN\cN(W,L)$ parameterized as (\ref{NN standard form}) can be written in the form (\ref{NN special form}) with
\[
\tilde{A}_L = (A_L,\boldsymbol{0}), \quad
\tilde{A}_\ell =
\begin{pmatrix}
A_\ell & \Bb_\ell \\
\boldsymbol{0} & 1
\end{pmatrix}
,\ \ell = 0,\dots,L-1,
\]
and
\begin{equation}\label{norm relation}
\prod_{\ell=0}^L \|\tilde{A}_\ell\| = \|A_L\| \prod_{\ell =0}^{L-1} \max\{\| (A_\ell,\Bb_\ell)\|,1\}.
\end{equation}
Hence, we define the norm constrained neural network $\cN\cN(W,L,K)$ as the set of functions $\phi_\theta \in \cN\cN(W,L)$ of the form (\ref{NN standard form}) that satisfies the following norm constraint on the weights
\begin{equation}\label{norm constraint}
\kappa(\theta) := \|A_L\| \prod_{\ell =0}^{L-1} \max\{\| (A_\ell,\Bb_\ell)\|,1\} \le K.
\end{equation}
The following proposition summarizes the relation between the two neural network classes $\cN\cN(W,L,K)$ and $\cS\cN\cN(W,L,K)$. It shows that we can essentially regard these two classes as the same when studying their expressiveness.

\begin{proposition}\label{network class relation}
$\cS\cN\cN(W,L,K) \subseteq \cN\cN(W,L,K) \subseteq \cS\cN\cN(W+1,L,K)$.
\end{proposition}
\begin{proof}
By the definition (\ref{norm constraint}) and the relation (\ref{norm relation}), it is easy to see that $\cN\cN(W,L,K) \subseteq \cS\cN\cN(W+1,L,K)$. Conversely, for any $\tilde{\phi}\in \cS\cN\cN(W,L,K)$ of the form (\ref{NN special form}), by the absolute homogeneity of the ReLU function, we can always rescale $\tilde{A}_\ell$ such that $\|\tilde{A}_L\| \le K$ and $\|\tilde{A}_\ell\|=1$ for $\ell \neq L$. Since the function $\tilde{\phi}$ can also be parameterized in the form (\ref{NN standard form}) with $\theta = ( \tilde{A}_0,  (\tilde{A}_1,\boldsymbol{0}),\dots,(\tilde{A}_{L-1},\boldsymbol{0}), \tilde{A}_L)$ and $\kappa(\theta) = \prod_{\ell=0}^L \|\tilde{A}_\ell\|\le K$, we have $\tilde{\phi}\in \cN\cN(W,L,K)$.
\end{proof}

The sample complexity of $\cS\cN\cN(W,L,K)$ has been studied in the recent works \citep{neyshabur2015norm,neyshabur2018pac,bartlett2017spectrally,golowich2020size}. By Proposition \ref{network class relation}, these sample complexity bounds can also be applied to $\cN\cN(W,L,K)$. We will use the Rademacher complexity to derive lower bounds for the approximation capacity of norm constrained neural networks.

\begin{definition}[Rademacher complexity]
Given a set $S\subseteq \bR^n$, the Rademacher complexity of $S$ is denoted by
\[
\cR_n(S) := \bE_{\xi_{1:n}} \left[ \sup_{(s_1,\dots,s_n)\in S} \frac{1}{n} \sum_{i=1}^n \xi_i s_i  \right],
\]
where $\xi_{1:n} = \{\xi_i\}_{i=1}^n$ is a sequence of i.i.d. Rademacher random variables which take the values $1$ and $-1$ with equal probability $1/2$.
\end{definition}

\begin{lemma}\label{Rademacher bound}
For any $\Bx_1,\dots,\Bx_n \in [-B,B]^d$ with $B\ge 1$, let $S:= \{(\phi(\Bx_1),\dots,\phi(\Bx_n)) :\phi \in \cS\cN\cN_{d,1}(W,L,K) \} \subseteq \bR^n$, then
\[
\cR_n(S) \le \frac{1}{n} K\sqrt{2(L+2+\log(d+1))} \max_{1\le j\le d+1} \sqrt{ \sum_{i=1}^n x_{i,j}^2} \le \frac{B K\sqrt{2(L+2+\log(d+1))}}{\sqrt{n}},
\]
where $x_{i,j}$ is the $j$-th coordinate of the vector $\tilde{\Bx}_i=(\Bx_i^\intercal,1)^\intercal\in \bR^{d+1}$. When $W\ge 2$,
\[
\cR_n(S) \ge \frac{K}{2\sqrt{2}n} \max_{1\le j\le d+1} \sqrt{ \sum_{i=1}^n x_{i,j}^2} \ge \frac{K}{2\sqrt{2n}}.
\]
\end{lemma}
\begin{proof}
The upper bound is from \citet[Theorem 3.2]{golowich2020size}.

For the lower bound, we consider the linear function class $\cF :=\{\Bx\mapsto \Ba^\intercal \tilde{\Bx}: \Ba\in \bR^{d+1}, \|\Ba\|_1\le K/2 \}$. Observing that $\Ba^\intercal \tilde{\Bx} = \sigma(\Ba^\intercal \tilde{\Bx}) - \sigma(-\Ba^\intercal \tilde{\Bx})$, we conclude that $\cF\subseteq \cS\cN\cN(2,1,K) \subseteq \cS\cN\cN(W,L,K)$, where the last inclusion follows from Proposition \ref{basic construct}. Therefore,
\begin{align*}
\cR_n(S) &= \frac{1}{n} \bE_{\xi_{1:n}} \left[ \sup_{\phi\in \cS\cN\cN(W,L,K)} \sum_{i=1}^n \xi_i \phi(\Bx_i) \right] \ge \frac{1}{n} \bE_{\xi_{1:n}} \left[ \sup_{\|\Ba\|_1\le K/2} \sum_{i=1}^n \xi_i \Ba^\intercal \tilde{\Bx}_i \right] \\
& = \frac{K}{2n} \bE_{\xi_{1:n}} \left\| \sum_{i=1}^n \xi_i \tilde{\Bx}_i \right\|_\infty = \frac{K}{2n} \bE_{\xi_{1:n}} \max_{1\le j\le d+1} \left| \sum_{i=1}^n \xi_i x_{i,j} \right| \\
&\ge \frac{K}{2n} \max_{1\le j\le d+1} \bE_{\xi_{1:n}} \left| \sum_{i=1}^n \xi_i x_{i,j} \right| \\
&\ge \frac{K}{2\sqrt{2}n} \max_{1\le j\le d+1} \sqrt{ \sum_{i=1}^n x_{i,j}^2},
\end{align*}
where the last inequality is due to Khintchine inequality, see \citet[Lemma 4.1]{ledoux1991probability} and \citet{haagerup1981best}.
\end{proof}

The next proposition shows that we can always normalize the weights of $\phi \in \cN\cN(W,L,K)$ such that the norm of each weight matrix in the hidden layers is at most one.

\begin{proposition}[Rescaling]\label{normalize}
Every $\phi\in \cN\cN(W,L,K)$ can be written in the form (\ref{NN standard form}) such that $\|A_L\|\le K$ and $\| (A_\ell,\Bb_\ell)\|\le 1$ for $0\le \ell \le L-1$.
\end{proposition}
\begin{proof}
We first parameterize $\phi$ in the form (\ref{NN standard form}) and denote $k_\ell:=\max \{\| (A_\ell,\Bb_\ell)\|,1\}$ for all $0\le \ell \le L-1$. We let $\tilde{A}_\ell = A_\ell/k_\ell$, $\tilde{\Bb}_\ell = \Bb_\ell/(\prod_{i=0}^{\ell}k_i)$, $\tilde{A}_L=A_L \prod_{i=0}^{L-1}k_i$ and consider the new parameterization of $\phi$:
\[
\tilde{\phi}_{\ell+1}(\Bx) = \sigma(\tilde{A}_\ell \tilde{\phi}_\ell(\Bx)+\tilde{\Bb}_\ell), \quad \tilde{\phi}_0(\Bx)=\Bx.
\]
It is easy to check that $\|\tilde{A}_L\|\le K$ and
\[
\| (\tilde{A}_\ell, \tilde{\Bb}_\ell) \| = \frac{1}{k_\ell} \left\| \left(A_\ell, \frac{\Bb_\ell}{\prod_{i=0}^{\ell-1}k_i} \right) \right\| \le \frac{1}{k_\ell} \| (A_\ell, \Bb_\ell) \| \le 1,
\]
where the second inequality is due to $k_i\ge 1$.

Next, we show that $\phi_\ell(\Bx) = \left( \prod_{i=0}^{\ell-1} k_i \right) \tilde{\phi}_{\ell}(\Bx)$ by induction. For $\ell =1$, by the absolute homogeneity of the ReLU function,
\[
\phi_1(\Bx) = \sigma(A_0 \Bx+\Bb_0) = k_0 \sigma(\tilde{A}_0 \Bx+\tilde{\Bb}_0) = k_0 \tilde{\phi}_1(\Bx).
\]
Inductively, one can conclude that
\begin{align*}
\phi_{\ell+1}(\Bx) &= \sigma(A_{\ell} \phi_{\ell}(\Bx)+\Bb_\ell) = \left( \prod_{i=0}^\ell k_i \right) \sigma\left( \tilde{A}_\ell \frac{\phi_\ell(\Bx)}{\prod_{i=0}^{\ell-1} k_i} + \tilde{\Bb}_\ell \right) \\
&= \left( \prod_{i=0}^\ell k_i \right)  \sigma\left( \tilde{A}_\ell \tilde{\phi}_\ell(\Bx) + \tilde{\Bb}_\ell \right) = \left( \prod_{i=0}^\ell k_i \right) \tilde{\phi}_{\ell+1}(\Bx),
\end{align*}
where the third equality is due to induction. Therefore,
\[
\phi(\Bx) = A_L\phi_L(\Bx) = A_L \left( \prod_{i=0}^{L-1} k_i \right) \tilde{\phi}_{L}(\Bx) = \tilde{A}_L \tilde{\phi}_L (\Bx),
\]
which means $\phi$ can be parameterized by $((\tilde{A}_0,\tilde{\Bb}_0),\dots,(\tilde{A}_{L-1},\tilde{\Bb}_{L-1}), \tilde{A}_L)$ and we finish the proof.
\end{proof}

In the following proposition, we summarize some basic operations on neural networks. These operations will be useful for the construction of neural networks, when we study the approximation capacity.

\begin{proposition}\label{basic construct}
Let $\phi_1\in \cN\cN_{d_1,k_1}(W_1,L_1,K_1)$ and $\phi_2\in \cN\cN_{d_2,k_2}(W_2,L_2,K_2)$.
\begin{enumerate}[label=\textnormal{(\roman*)},parsep=0pt]
\item If $d_1=d_2$, $k_1=k_2$, $W_1\le W_2$, $L_1\le L_2$ and $K_1\le K_2$, then $\cN\cN_{d_1,k_1}(W_1,L_1,K_1) \subseteq \cN\cN_{d_2,k_2}(W_2,L_2,K_2)$.

\item \textnormal{\textbf{(Composition)}} If $k_1 = d_2$, then $\phi_2 \circ \phi_1 \in \cN\cN_{d_1,k_2}(\max\{W_1,W_2\},L_1+L_2, K_2\max\{K_1,1\})$.
Let $A\in \bR^{d_2\times d_1}$ and $\Bb\in \bR^{d_2}$. Define the function $\phi(\Bx) :=\phi_2(A\Bx+\Bb)$ for $\Bx\in \bR^{d_1}$, then $\phi\in \cN\cN_{d_1,k_2}(W_2,L_2, K_2 \max\{\|(A,\Bb)\|,1\})$.

\item \textnormal{\textbf{(Concatenation)}} If $d_1=d_2$, define $\phi(\Bx):=(\phi_1(\Bx),\phi_2(\Bx))$, then $\phi\in \cN\cN_{d_1,k_1+k_2}(W_1+W_2,\max\{L_1,L_2\},\max\{K_1,K_2\})$.

\item \textnormal{\textbf{(Linear Combination)}} If $d_1=d_2$ and $k_1=k_2$, then, for any $c_1,c_2\in\bR$, $c_1\phi_1 + c_2\phi_2 \in \cN\cN_{d_1,k_1}(W_1+W_2, \max\{L_1,L_2\}, |c_1|K_1+|c_2|K_2)$.
\end{enumerate}
\end{proposition}

\begin{proof}
By Proposition \ref{normalize}, we can parameterize $\phi_i$, $i=1,2$, in the form (\ref{NN standard form}) with parameters $((A^{(i)}_0,\Bb^{(i)}_0),\dots,(A^{(i)}_{L_i-1},\Bb^{(i)}_{L_i-1}),A^{(i)}_{L_i})$ such that $\| A^{(i)}_{L_i}\| \le K_i$ and $\| (A^{(i)}_\ell, \Bb^{(i)}_\ell)\| \le 1$ for $\ell \neq L_i$.

(i) We can assume that $A^{(1)}_\ell \in \bR^{W_2\times W_2}$ and $\Bb^{(1)}_\ell \in \bR^{W_2}$, $0\le \ell \le L_1-1$, by adding suitable zero rows and columns to $A^{(1)}_\ell$ and $\Bb^{(1)}_\ell$ if necessary (this operation does not change the norm). Then, $\phi_1$ can also be parameterized by the parameters
\[
\left( \left(A^{(1)}_0, \Bb^{(1)}_0\right),\dots, \left(A^{(1)}_{L_1-1},\Bb^{(1)}_{L_1-1}\right), \underbrace{\left(\Id,\boldsymbol{0} \right),\dots, \left(\Id,\boldsymbol{0} \right)}_{L_2-L_1 \mbox{ times }}, A^{(1)}_{L_1} \right),
\]
where $\Id$ is the identity matrix. Hence, $\phi_1 \in \cN\cN_{d_2,k_2}(W_2,L_2,K_2)$.

(ii) By (i), we can assume $W_1=W_2$ without loss of generality. Then, $\phi_2 \circ \phi_1$ can be parameterized by
\[
\left( \left(A^{(1)}_0, \Bb^{(1)}_0\right),\dots, \left(A^{(1)}_{L_1-1},\Bb^{(1)}_{L_1-1}\right), \left(A^{(2)}_0 A^{(1)}_{L_1}, \Bb^{(2)}_0\right), \left(A^{(2)}_1, \Bb^{(2)}_1\right),\dots, \left(A^{(2)}_{L_2-1},\Bb^{(2)}_{L_2-1}\right), A^{(2)}_{L_2} \right).
\]
We observe that
\[
\left\| \left(A^{(2)}_0 A^{(1)}_{L_1}, \Bb^{(2)}_0\right)\right\| = \left\| \left( A^{(2)}_0, \Bb^{(2)}_0 \right)
\begin{pmatrix}
A^{(1)}_{L_1} & \boldsymbol{0} \\
\boldsymbol{0} & 1
\end{pmatrix} \right\|
\le \left\| \left( A^{(2)}_0, \Bb^{(2)}_0 \right) \right\| \left\|
\begin{pmatrix}
A^{(1)}_{L_1} & \boldsymbol{0} \\
\boldsymbol{0} & 1
\end{pmatrix} \right\|
\le \max\{K_1,1\}.
\]
Hence, $\phi_2 \circ \phi_1 \in \cN\cN_{d_1,k_2}(W_1,L_1+L_2, K_2\max\{K_1,1\})$.

For the function $\phi(\Bx) :=\phi_2(A\Bx+\Bb)$, we can similarly parameterize it by
\[
\left( \left(A^{(2)}_0 A, A^{(2)}_0 \Bb + \Bb^{(2)}_0\right), \left(A^{(2)}_1, \Bb^{(2)}_1\right),\dots, \left(A^{(2)}_{L_2-1},\Bb^{(2)}_{L_2-1}\right), A^{(2)}_{L_2} \right).
\]
Using
\[
\left\| \left(A^{(2)}_0 A, A^{(2)}_0 \Bb + \Bb^{(2)}_0\right) \right\| = \left\| \left( A^{(2)}_0, \Bb^{(2)}_0 \right)
\begin{pmatrix}
A & \Bb \\
\boldsymbol{0} & 1
\end{pmatrix} \right\|
\le \max\{\|(A,\Bb)\|,1\},
\]
we conclude that $\phi\in \cN\cN(W_2,L_2, K_2 \max\{\|(A,\Bb)\|,1\})$.

(iii) By (i), we can assume that $L_1=L_2$. Then, $\phi$ can be parameterized by the parameters $((A_0,\Bb_0),\dots,(A_{L_1-1},\Bb_{L_1-1}),A_{L_1})$ where
\[
A_0 :=
\begin{pmatrix}
A^{(1)}_0   \\
A^{(2)}_0
\end{pmatrix},
\Bb_0 :=
\begin{pmatrix}
\Bb^{(1)}_0 \\
\Bb^{(2)}_0
\end{pmatrix},
\quad A_\ell :=
\begin{pmatrix}
A^{(1)}_\ell & \boldsymbol{0}  \\
\boldsymbol{0} & A^{(2)}_\ell
\end{pmatrix},
\Bb_\ell :=
\begin{pmatrix}
\Bb^{(1)}_\ell \\
\Bb^{(2)}_\ell
\end{pmatrix},
\ell \neq 0.
\]
Notice that $\|A_{L_1}\| = \max\{ \|A^{(1)}_{L_1}\|, \|A^{(2)}_{L_1}\| \}\le \max\{K_1,K_2 \}$ and
\begin{align*}
\| (A_0,\Bb_0)\| &= \left\|
\begin{pmatrix}
A^{(1)}_0 &  \Bb^{(1)}_0  \\
A^{(2)}_0 & \Bb^{(2)}_0
\end{pmatrix} \right\| \le 1,\\
\| (A_\ell,\Bb_\ell)\| &= \left\|
\begin{pmatrix}
A^{(1)}_\ell & \boldsymbol{0} & \Bb^{(1)}_\ell  \\
\boldsymbol{0} & A^{(2)}_\ell & \Bb^{(2)}_\ell
\end{pmatrix} \right\| \le 1, \quad 0<\ell<L_1.
\end{align*}

(iv) Replacing the matrix $A_{L_1}$ in (iii) by $(c_1A^{(1)}_{L_1}, c_2A^{(2)}_{L_1} )$, the conclusion follows from
\[
\left\| \left(c_1A^{(1)}_{L_1}, c_2A^{(2)}_{L_1} \right)\right\|\le |c_1| \left\|A^{(1)}_{L_1}\right\| + |c_2| \left\|A^{(2)}_{L_1}\right\|\le |c_1|K_1+|c_2|K_2. \qedhere
\]
\end{proof}

\section{Approximation of smooth functions}\label{sec: approximation}

In this section, we study how well norm constrained neural networks approximate smooth functions. To begin with, let us introduce the notion of regularity of functions.

\begin{definition}[H\"older classes]
Let $d\in \bN$ and $\alpha = r+\beta>0$, where $r\in \bN_0 $ and $\beta\in (0,1]$. We denote the H\"older class $\cH^\alpha(\bR^d)$ as
\[
\cH^\alpha(\bR^d) := \left\{ f:\bR^d\to \bR, \max_{\|\Bs\|_1\le r} \sup_{\Bx\in \bR^d} |\partial^\Bs f(\Bx)| \le 1, \max_{\|\Bs\|_1=r} \sup_{\Bx\neq \By} \frac{|\partial^\Bs f(\Bx)- \partial^\Bs f(\By)|}{\|\Bx-\By\|_\infty^\beta}\le 1 \right\},
\]
where the multi-index $\Bs\in \bN_0^d$. Denote $\cH^\alpha := \{ f:[0,1]^d\to \bR, f\in \cH^\alpha(\bR^d) \}$ as the restriction of $\cH^\alpha(\bR^d)$ to $[0,1]^d$.
\end{definition}

It should be noticed that for $\alpha=r+1$, we do not assume that $f\in C^{r+1}$. Instead, we only require that $f\in C^r$ and its derivatives of order $r$ are Lipschitz continuous. In particular, when $\alpha = 1$, $\cH^1$ is the set of bounded $1$-Lipschitz continuous functions:
\[
\|f\|_{L^\infty} \le 1 \mbox{ and } \Lip (f) := \sup_{\Bx\neq \By} \frac{|f(\Bx)-f(\By)|}{\|\Bx-\By\|_\infty} \le 1.
\]
We will also denote $\Lip 1:= \{f:\Lip(f)\le 1 \}$ for convenience. Thus, $\cH^1 \subseteq \Lip 1$.

Since the ReLU function is $1$-Lipschitz, it is easy to see that, for any $\phi_\theta\in \cN\cN(W,L,K)$,
\[
\Lip (\phi_\theta) \le \kappa(\theta) \le K.
\]
However, it was shown by \citet{huster2019limitations} that some simple $1$-Lipschitz functions, such as $f(x) = |x|$, can not be represented by $\cN\cN(W,L,K)$ for any $K<2$. Their result implies that norm constrained neural networks have a restrictive expressive power. Nevertheless, since two-layer neural networks are universal, $\cN\cN(W,L,K)$ can approximate any continuous functions when $W$ and $K$ are sufficiently large. In the following, we will try to quantify the approximation error
\[
\cE(\cH^\alpha, \cN\cN(W,L,K)) := \sup_{f\in \cH^\alpha} \inf_{\phi \in \cN\cN(W,L,K)} \| f- \phi\|_{C ([0,1]^d)},
\]
where $C ([0,1]^d)$ is the space of continuous functions on $[0,1]^d$ equipped with the sup-norm. Our main results can be summarized in the following theorem.

\begin{theorem}\label{app bounds}
Let $d\in \bN$ and $\alpha = r+\beta>0$, where $r\in \bN_0 $ and $\beta\in (0,1]$.
\begin{enumerate}[label=\textnormal{(\arabic*)},parsep=0pt]

\item There exists $c>0$ such that for any $K\ge 1$, any $W\ge c K^{(2d+\alpha)/(2d+2)}$ and $L \ge 2\lceil \log_2 (d+r) \rceil+2$,
\[
\cE(\cH^\alpha, \cN\cN(W,L,K)) \lesssim K^{-\alpha/(d+1)}.
\]

\item If $d>2\alpha$, then for any $W,L \in \bN$, $W\ge 2$ and $K \ge 1$,
\[
\cE(\cH^\alpha, \cN\cN(W,L,K)) \gtrsim (K \sqrt{L})^{-2\alpha/(d-2\alpha)}.
\]
\end{enumerate}
\end{theorem}

We note that the (implied) constants in the theorem only depend on $d$ and $\alpha$. We also note that the lower bound is derived from the upper bound of Rademacher complexity in Lemma \ref{Rademacher bound}, which is independent of the width $W$. Notice that the lower bound of Rademacher complexity in Lemma \ref{Rademacher bound} is also independent of the depth $L$. When assuming more control over Schatten norm of the parameter matrices, \citet{golowich2020size} obtained sample complexity upper bounds that are independent of the size of neural networks. Consequently, one can obtain size-independent lower bound of approximation error for such neural networks.

\subsection{Upper bounds}

The upper bound in Theorem \ref{app bounds} is proved by an explicit construction of norm constrained neural networks that approximate the local Taylor polynomials. Following the constructions in \citep{yarotsky2017error,yarotsky2018optimal,yarotsky2020phase,lu2021deep}, we first consider the approximation of the quadratic function $f(x)=x^2$ and then extend the approximation to monomials.

\begin{lemma}\label{square}
For any $k\in \bN$, there exists $\phi_k \in \cN\cN(k,1,3)$ such that $\phi_k(x)=0$ for $x\le 0$, $\phi_k(x)\in[0,1]$ for $x\in [0,1]$ and
\[
\left|x^2 - \phi_k(x) \right| \le \frac{1}{2k^2}, \quad x\in [0,1].
\]
\end{lemma}
\begin{proof}
The construction is based on the integral representation of $x^2$:
\begin{equation}\label{x^2 representation}
x^2 = \int_0^x 2x-2b db = \int_0^x 2\sigma(x-b) db = \int_0^1 2\sigma(x-b) db, \quad x\in [0,1].
\end{equation}
We can approximate the integral by Riemann sum. For any $k\in \bN$, define
\[
\phi_k(x) = \frac{1}{k} \sum_{i=1}^k 2 \sigma \left(x-\frac{2i-1}{2k} \right).
\]
Then, by Proposition \ref{basic construct}, $\phi_k\in \cN\cN(k,1,K)$ with
\[
K =  \sum_{i=1}^k \frac{2}{k} \left(1+\frac{2i-1}{2k}\right) = 3.
\]
It is easy to see that $\phi_k(x)=0$ for $x\le 0$. Since $\phi_k$ is an increasing function, we have $0=\phi_k(0)\le \phi_k(x) \le \phi_k(1) =1$ for $x\in[0,1]$.

For any $x\in(0,1]$, let us denote $i_x = \lceil kx \rceil \in \{1,\dots,k\}$, then $x\in ((i_x-1)/k,i_x/k]$. If $i<i_x$, then
\[
\int_{(i-1)/k}^{i/k} 2\sigma(x-b) db = \int_{(i-1)/k}^{i/k} 2x-2b db = \frac{2x}{k} - \frac{2i-1}{k^2} = \frac{2}{k} \sigma\left(x-\frac{2i-1}{2k} \right) .
\]
If $i>i_x$, then
\[
\int_{(i-1)/k}^{i/k} 2\sigma(x-b) db = 0 = \frac{2}{k} \sigma\left(x-\frac{2i-1}{2k} \right) .
\]
Therefore,
\begin{align*}
\left|x^2 - \phi_k(x)\right| &= \left| \sum_{i=1}^k \int_{(i-1)/k}^{i/k} 2\sigma(x-b) db -  \sum_{i=1}^k\frac{2}{k}\sigma \left(x-\frac{2i-1}{2k} \right) \right| \\
&= \left| \int_{(i_x-1)/k}^{i_x/k} 2\sigma(x-b)  - 2\sigma \left(x-\frac{2i_x-1}{2k} \right) db \right| \\
&\le \int_{(i_x-1)/k}^{i_x/k} 2 \left| b - \frac{2i_x-1}{2k} \right| db = \frac{1}{2k^2},
\end{align*}
where we use the Lipschitz continuity of ReLU in the inequality. 
\end{proof}

\begin{remark}
Our construction is based on the integral representation (\ref{x^2 representation}), which can be regarded as an infinite width neural network. This construction is different from the construction in \cite{yarotsky2017error}, which use the teeth function $T_{i} = T_1 \circ T_{i-1} = T_1 \circ \cdots \circ T_1$ to construct the approximator
\[
f_k(x) = x- \sum_{i=1}^k 4^{-i} T_i(x),
\]
where $T_1(x) = 2x$ for $x\in [0,1/2]$ and $T_1(x) = 2(1-x)$ for $x\in [1/2,1]$. It can be shown that $f_k$ achieves the approximation error $|x^2-f_k(x)|\le 2^{-2(k+1)}$. 
Since $T_1\in \cN\cN(2,2,7)$, by Proposition \ref{basic construct}, this compositional property implies $T_i\in \cN\cN(2,2i,7^i)$ and consequently one can show that $f_k \in \cN\cN(2k+1,2k,\frac{4}{3}(\frac{7}{4})^{k+1}-\frac{4}{3})$. Hence, in the construction of \citet{yarotsky2017error}, the approximation error decays exponentially with the depth but only polynomially with the norm constraint $K$. On the contrary, in our construction, the network has a finite norm constraint but the approximation error decays only quadratically on the width.
\end{remark}

Using the relation $xy = 2\left((\frac{x+y}{2})^2- (\frac{x}{2})^2 - (\frac{y}{2})^2 \right)$, we can approximate the product function by neural networks and then further approximate any monomials $x_1\cdots x_d$.

\begin{lemma}\label{product}
For any $k\in \bN$, there exists $\psi_k \in \cN\cN(6k,2,216)$ such that $\psi_k:[-1,1]^2 \to [-1,1]$ and
\[
|xy - \psi_k(x,y)| \le \frac{3}{k^2}, \quad x,y\in [-1,1].
\]
Furthermore, $\psi_k(x,y)=0$ if $xy=0$.
\end{lemma}
\begin{proof}
Let $\phi_k \in \cN\cN(k,1,3)$ be the network in Lemma \ref{square} and define $\widetilde{\phi}_k(x) = \phi_k(x)+\phi_k(-x)$. By Proposition \ref{basic construct}, $\widetilde{\phi}_k \in \cN\cN(2k,1,6)$. Since $\phi_k(x)=0$ for $x\le 0$, we have $\widetilde{\phi}_k(x) = \phi_k(|x|)$ and the approximation error is 
\[
\left| x^2 - \widetilde{\phi}_k(x) \right| = \left| x^2 - \phi_k(|x|) \right| \le \frac{1}{2k^2}, \quad x\in [-1,1].
\]
Using the fact that $xy = 2\left((\frac{x+y}{2})^2- (\frac{x}{2})^2 - (\frac{y}{2})^2 \right)$, we consider the function
\[
\widetilde{\psi}_k(x,y) := 2\widetilde{\phi}_k\left(\frac{1}{2}x + \frac{1}{2}y \right) -  2\widetilde{\phi}_k\left(\frac{1}{2}x \right) -  2\phi_k\left(\frac{1}{2}y\right).
\]
Then, $\widetilde{\psi}_k(x,y)=0$ if $xy=0$, and, for any $x,y\in[-1,1]$,
\[
\left|xy - \widetilde{\psi}_k(x,y)\right| \le 2\left| \left(\frac{x+y}{2}\right)^2 - \widetilde{\phi}_k\left(\frac{x+y}{2}\right)\right| + 2\left| \left(\frac{x}{2}\right)^2 - \widetilde{\phi}_k\left(\frac{x}{2}\right)\right| + 2\left| \left(\frac{y}{2}\right)^2 - \widetilde{\phi}_k\left(\frac{y}{2}\right)\right| \le \frac{3}{k^2}.
\]
By Proposition \ref{basic construct}, $\widetilde{\psi}_k\in \cN\cN(6k, 1, 36)$.

Finally, let $\chi(x) = \sigma(x) - \sigma(-x) -2\sigma(\tfrac{1}{2}x-\tfrac{1}{2}) +2\sigma(-\tfrac{1}{2}x-\tfrac{1}{2})  = (x \lor -1) \land 1$, then $\chi \in \cN\cN(4,1,6)$. We construct the target function as
\[
\psi_k(x,y) = \chi (\widetilde{\psi}_k(x,y)) = (\widetilde{\psi}_k(x,y) \lor -1) \land 1.
\]
Then, for any $x,y\in[-1,1]$,
\[
|xy - \psi_k(x,y)| \le |xy - \widetilde{\psi}_k(x,y)| \le \frac{3}{k^2}.
\]
By Proposition \ref{basic construct}, $\psi_k \in \cN\cN(6k, 2, 216)$.
\end{proof}

\begin{lemma}\label{d product}
For any $d\ge 2$ and $k\in \bN$ , there exists $\phi \in \cN\cN(6d k, 2\lceil \log_2 d \rceil,6^{3\lceil \log_2 d \rceil})$ such that $\phi:[-1,1]^d \to [-1,1]$ and
\[
|x_1\cdots x_d - \phi(\Bx)| \le \frac{6d}{k^2}, \quad \Bx=(x_1,\dots,x_d)^\intercal \in [-1,1]^d.
\]
Furthermore, $\phi(\Bx)=0$ if $x_1\cdots x_d=0$.
\end{lemma}
\begin{proof}
We firstly consider the case $d=2^m$ for some $m\in\bN$. For $m=1$, by Lemma \ref{product}, there exists $\phi_1 \in \cN\cN(6k,2,216)$ such that $\phi_1:[-1,1]^2 \to [-1,1]$ and $|x_1x_2 - \phi_1(x_1,x_2)| \le 3k^{-2}$ for any $x_1,x_2\in [-1,1]$. We define $\phi_m:[-1,1]^{2^m} \to [-1,1]$ inductively by
\[
\phi_{m+1}(x_1,\dots,x_{2^{m+1}}) = \phi_1(\phi_m(x_1,\dots,x_{2^m}),\phi_m(x_{2^m+1},\dots,x_{2^{m+1}})).
\]
Then, $\phi_m(x_1,\dots,x_{2^m}) =0$ if $x_1\cdots x_{2^m}=0$ because this equation is true for $m=1$. Next, we inductively show that $\phi_m \in \cN\cN(3k 2^m, 2m,216^m)$ and
\[
|x_1\cdots x_{2^m} - \phi_m(x_1,\dots,x_{2^m})| \le (2^m - 1)\epsilon.
\]
where we denote $\epsilon:= 3k^{-2}$, i.e. the approximation error of $\phi_1$.

It is obvious that the assertion is true for $m=1$ by construction. Assume that the assertion is true for some $m\in \bN$, we will prove that it is true for $m+1$. By Proposition \ref{basic construct} and the construction of $\phi_{m+1}$, we have $\phi_{m+1} \in \cN\cN(3k 2^{m+1}, 2m+2,216^{m+1})$. For any $x_1,\dots,x_{2^{m+1}} \in [-1,1]$, we denote $s_1 := x_1\cdots x_{2^m}$, $t_1:=x_{2^m+1}\cdots x_{2^{m+1}}$, $s_2:= \phi_m(x_1,\dots,x_{2^m})$ and $t_2:=\phi_m(x_{2^m+1},\dots,x_{2^{m+1}})$, then $s_1,t_1,s_2,t_2\in [-1,1]$. By the hypothesis of induction,
\[
|s_1 - s_2|, |t_1-t_2| \le (2^m - 1)\epsilon.
\]
Therefore,
\begin{align*}
&|x_1 \cdots x_{2^{m+1}} - \phi_{m+1}(x_1,\dots,x_{2^{m+1}})| 
= |s_1t_1 - \phi_1(s_2,t_2)| \\
\le & |s_1t_1 - s_1t_2| + |s_1t_2 - s_2t_2| + |s_2t_2 - \phi_1(s_2,t_2)| \\
\le & |t_1 - t_2| + |s_1 - s_2| + \epsilon 
\le  (2^{m+1} - 1)\epsilon.
\end{align*}
Hence, the assertion is true for $m+1$.

For general $d\ge 2$, we choose $m=\lceil \log_2 d \rceil$, then $2^{m-1} < d \le 2^m$. We define the target function $\phi:[-1,1]^d \to [-1,1]$ by
\[
\phi(\Bx):= \phi_m \left(
\begin{pmatrix}
\Id_d \\
\boldsymbol{0}_{(2^m-d)\times d}
\end{pmatrix} \Bx +
\begin{pmatrix}
\boldsymbol{0}_{d\times 1} \\
\boldsymbol{1}_{(2^m-d)\times 1}
\end{pmatrix}
\right),
\]
where $\Id_d$ is the $d\times d$ identity matrix, $\boldsymbol{0}_{p\times q}$ is the $p\times q$ zero matrix and $\boldsymbol{1}_{(2^m-d)\times 1}$ is an all ones vector. By Proposition \ref{basic construct}, $\phi \in \cN\cN(3k 2^m, 2m,216^m) \subseteq \cN\cN(6d k, 2\lceil \log_2 d \rceil,6^{3\lceil \log_2 d \rceil})$ and the approximation error is
\[
|x_1\cdots x_d - \phi(\Bx)| \le (2^m - 1)\epsilon \le 2d \epsilon = 6d k^{-2}.
\]
Furthermore, $\phi(\Bx)=0$ if $x_1\cdots x_d=0$ because $\phi_m$ has such property.
\end{proof}

In Lemma \ref{d product}, we constructed neural networks to approximate monomials. We can then approximate any $f\in \cH^\alpha$ by approximating its local Taylor expansion
\begin{equation}\label{taylor}
p(\Bx) = \sum_{\Bn\in \{0,1,\dots,N\}^d} \psi_\Bn(\Bx) \sum_{\|\Bs\|_1\le r} \frac{\partial^\Bs f(\frac{\Bn}{N})}{\Bs !} \left(\Bx - \frac{\Bn}{N} \right)^\Bs,
\end{equation}
where we use the usual conventions $\Bs ! = \prod_{i=1}^d s_i !$ and $(\Bx - \frac{\Bn}{N} )^\Bs = \prod_{i=1}^d (x_i-\frac{n_i}{N})^{s_i}$. The functions $\{\psi_\Bn\}_{\Bn}$ form a partition of unity of $[0,1]^d$ and each $\psi_\Bn$ is supported on a sufficiently small neighborhood of $\Bn/N$.

\begin{theorem}\label{upper bound}
For any $N,k\in \bN$ and $h\in \cH^\alpha$ with $\alpha = r+\beta$, where $r\in \bN_0 $ and $\beta\in (0,1]$, there exists $\phi\in \cN\cN(W,L,K)$ where
\begin{align*}
W &= 6(r+1)(d+r) d^r (N+1)^d k, \\
L &= 2\lceil \log_2 (d+r) \rceil + 2, \\
K &= 6^{3\lceil \log_2 (d+r) \rceil+1}(r+1)d^rN(N+1)^d,
\end{align*}
such that
\[
\| h-\phi \|_{L^\infty([0,1]^d)} \le 2^d d^r(N^{-\alpha} + 6(r+1) (d+r) k^{-2} ).
\]
\end{theorem}

\begin{proof}
Let
\[
\psi(t) = \sigma(1-|t|) = \sigma(1-\sigma(t)-\sigma(-t)) \in [0,1], \quad t\in \bR,
\]
then $\psi \in \cN\cN(2,2,3)$ and the support of $\psi$ is $[-1,1]$. For any $\Bn=(n_1,\dots,n_d)\in \{0,1,\dots,N\}^d$, define
\[
\psi_\Bn(\Bx) := \prod_{i=1}^{d} \psi(Nx_i-n_i), \quad \Bx=(x_1,\dots,x_d)^\intercal \in \bR^d,
\]
then $\psi_\Bn$ is supported on $\{\Bx\in \bR^d: \|\Bx-\tfrac{\Bn}{N}\|_\infty \le \tfrac{1}{N} \}$. The functions $\{\psi_\Bn\}_\Bn$ form a partition of unity of the domain $[0,1]^d$:
\[
\sum_{\Bn\in \{0,1,\dots,N\}^d} \psi_\Bn(\Bx) = \prod_{i=1}^{d} \sum_{n_i=0}^N \psi(Nx_i-n_i) \equiv 1, \quad \Bx\in [0,1]^d.
\]

Let $p(\Bx)$ be the local Taylor expansion (\ref{taylor}). For convenience, we denote $p_{\Bn,\Bs}(\Bx):= \psi_\Bn(\Bx) (\Bx-\frac{\Bn}{N})^\Bs$ and $c_{\Bn,\Bs}:=\partial^\Bs h(\frac{\Bn}{N})/\Bs !$. Then, $p_{\Bn,\Bs}$ is supported on $\{\Bx\in \bR^d: \|\Bx-\tfrac{\Bn}{N}\|_\infty \le \tfrac{1}{N} \}$ and
\[
p(\Bx) = \sum_{\Bn\in \{0,1,\dots,N\}^d} \sum_{\|\Bs\|_1\le r} c_{\Bn,\Bs} p_{\Bn,\Bs}(\Bx).
\]
Using Taylor's Theorem with integral remainder (see \citet[Lemma A.8]{petersen2018optimal} for example), it can be shown that the approximation error is
\begin{align*}
|f(\Bx) -p(\Bx)| &= \left| \sum_\Bn \psi_\Bn(\Bx) f(\Bx) - \sum_{\Bn} \psi_\Bn(x) \sum_{\|\Bs\|_1\le r} c_{\Bn,\Bs} \left(\Bx - \frac{\Bn}{N} \right)^\Bs \right| \\
&\le \sum_\Bn \psi_\Bn(\Bx) \left| f(\Bx) - \sum_{\|\Bs\|_1\le r} c_{\Bn,\Bs} \left(\Bx - \frac{\Bn}{N} \right)^\Bs \right| \\
&\le \sum_{\Bn: \|\Bx-\tfrac{\Bn}{N}\|_\infty < \tfrac{1}{N}} \left| f(\Bx) - \sum_{\|\Bs\|_1\le r} c_{\Bn,\Bs} \left(\Bx - \frac{\Bn}{N} \right)^\Bs \right| \\
&\le \sum_{\Bn: \|\Bx-\tfrac{\Bn}{N}\|_\infty < \tfrac{1}{N}} d^r \left\| \Bx - \frac{\Bn}{N} \right\|_\infty^\alpha \\
&\le 2^d d^r N^{-\alpha}.
\end{align*}

Let $\Phi_D\in \cN\cN(6D k, 2\lceil \log_2 D \rceil,6^{3\lceil \log_2 D \rceil})$ be the $D$-product function constructed in Lemma \ref{d product}. Then, we can approximate $p_{\Bn,\Bs}$ by
\[
\phi_{\Bn,\Bs}(\Bx) := \Phi_{d+\|\Bs\|_1}(\psi(Nx_1-n_1),\dots,\psi(Nx_d-n_d),\dots,x_i-\tfrac{n_i}{N},\dots),
\]
where the term $x_i-n_i/N$ appears in the input only when $s_i\neq 0$ and it repeats $s_i$ times. (When $d=1$ and $\Bs=\boldsymbol{0}$, we simply let $\phi_{n,\boldsymbol{0}}(x) =\psi(Nx-n) $.) Since $x_i-n_i/N = \sigma(x_i-n_i/N) - \sigma(-x_i+n_i/N)$ and $\|\Bs\|_1\le r$, by Proposition \ref{basic construct}, we have $\phi_{\Bn,\Bs}\in \cN\cN(6(d+r) k, 2\lceil \log_2 (d+r) \rceil + 2, 6^{3\lceil \log_2 (d+r) \rceil +1}N)$. By Lemma \ref{d product}, the approximation error is
\[
|p_{\Bn,\Bs}(\Bx) - \phi_{\Bn,\Bs}(\Bx)| \le 6(d+r) k^{-2}.
\]
Since $\Phi_D(t_1,\dots,t_D)=0$ when $t_1t_2\cdots t_D=0$, $\phi_{\Bn,\Bs}$ is supported on $\{\Bx\in \bR^d: \|\Bx-\tfrac{\Bn}{N}\|_\infty \le \tfrac{1}{N} \}$.

Now, we can approximate $p(\Bx)$ by
\[
\phi(\Bx) = \sum_{\Bn\in \{0,1,\dots,N\}^d} \sum_{\|\Bs\|_1\le r} c_{\Bn,\Bs} \phi_{\Bn,\Bs}(\Bx).
\]
Observe that $|c_{\Bn,\Bs}|=|\partial^\Bs f(\frac{\Bn}{N})/\Bs !|\le 1$ and the number of terms in the inner summation is
\[
\sum_{\|\Bs\|_1\le r} 1 = \sum_{j=0}^r \sum_{\|\Bs\|_1=j} 1 \le \sum_{j=0}^r d^j \le (r+1)d^r.
\]
The approximation error is, for any $\Bx\in [0,1]^d$,
\begin{align*}
|p(\Bx) - \phi(\Bx)| =& \left| \sum_\Bn \sum_{\|\Bs\|_1\le r} c_{\Bn,\Bs} p_{\Bn,\Bs}(\Bx) - \sum_\Bn \sum_{\|\Bs\|_1\le r} c_{\Bn,\Bs} \phi_{\Bn,\Bs}(\Bx) \right| \\
\le & \sum_\Bn \sum_{\|\Bs\|_1\le r} |c_{\Bn,\Bs}| |p_{\Bn,\Bs}(\Bx) - \phi_{\Bn,\Bs}(\Bx)| \\
\le& \sum_{\Bn: \|\Bx-\tfrac{\Bn}{N}\|_\infty < \tfrac{1}{N}} \sum_{\|\Bs\|_1\le r} |p_{\Bn,\Bs}(\Bx) - \phi_{\Bn,\Bs}(\Bx)| \\
\le & 6\cdot 2^d(r+1) (d+r) d^r k^{-2}.
\end{align*}
Hence, the total approximation error is
\[
|h(\Bx) - \phi(\Bx)| \le |h(\Bx) - p(\Bx)| + |p(\Bx) - \phi(\Bx)| \le 2^d d^r(N^{-\alpha} + 6(r+1) (d+r) k^{-2} ).
\]
Finally, by Proposition \ref{basic construct}, $\phi\in \cN\cN(6(r+1)(d+r) d^r (N+1)^d k, 2\lceil \log_2 (d+r) \rceil + 2, 6^{3\lceil \log_2 (d+r) \rceil+1}(r+1)d^rN(N+1)^d)$.
\end{proof}

Using the construction in Theorem \ref{upper bound}, we can give a proof of the approximation upper bound in Theorem \ref{app bounds}.

\begin{proof}[Proof of Theorem \ref{app bounds} (Upper bound)]
We choose $N= \lceil k^{2/\alpha} \rceil$ in the Theorem \ref{upper bound}, then it shows the existence of $\phi \in \cN\cN(W,L,K)$ with
\begin{align*}
W &= 6(r+1)(d+r) d^r (N+1)^d k \asymp k^{2d/\alpha+1}, \\
L &= 2\lceil \log_2 (d+r) \rceil + 2, \\
K &= 6^{3\lceil \log_2 (d+r) \rceil+1}(r+1)d^rN(N+1)^d \asymp k^{2(d+1)/\alpha},
\end{align*}
such that $\| h-\phi \|_{L^\infty([0,1]^d)} \le 2^d d^r(N^{-\alpha} + 6(r+1) (d+r) k^{-2} ) \lesssim k^{-2}$.
Therefore, $k \asymp K^{\alpha/(2d+2)}$, $W \asymp k^{2d/\alpha+1} \asymp K^{(2d+\alpha)/(2d+2)}$ and we have the approximation bound
\[
\| h-\phi \|_{L^\infty([0,1]^d)} \lesssim k^{-2} \lesssim K^{-\alpha/(d+1)}.
\]
Since increasing $W$ and $L$ can only decrease the approximation error, the bound holds for any $W\gtrsim K^{(2d+\alpha)/(2d+2)}$ and $L \ge 2 \lceil \log_2 (d+r) \rceil +2$.
\end{proof}

\subsection{Lower bounds}

In this section, we present two methods that give lower bounds for the approximation error using norm constrained neural networks. Both methods use the Rademacher complexity (Lemma \ref{Rademacher bound}) to lower bound the approximation capacity. The first method is inspired by \citet{maiorov1999degree}, which characterized the approximation order by pseudo-dimension (or VC dimension \citep{vapnik1971uniform}). This method compares the packing numbers of neural networks $\cN\cN(W,L,K)$ and the target function class $\cH^\alpha$ on a suitably chosen data set. The second method establishes the lower bound by finding a linear functional that distinguishes the approximator and target classes. Using the second method, we give explicit constant on the approximation lower bound in Theorem \ref{explicit lower bound}, but it only holds for $\cH^1$.

Let us begin with the estimation of the packing number of $\cH^\alpha$. We first construct a series of subsets $\cH^\alpha_N \subseteq \cH^\alpha$ with high complexity and simple structure. To this end, we choose a $C^\infty$ function $\psi:\bR^d\to [0,\infty)$ which satisfies $\psi(\boldsymbol{0})=1$ and $\psi(\Bx)=0$ for $\|\Bx\|_\infty\ge 1/4$, and let $C_{\psi,\alpha}>0$ be a constant such that $C_{\psi,\alpha} \psi \in \cH^\alpha(\bR^d)$. For any $N\in\bN$, we consider the function class
\begin{equation}\label{Halpha_N}
\cH^\alpha_N := \left\{ h_{\Ba}(\Bx) = \frac{C_{\psi,\alpha}}{N^{\alpha}} \sum_{\Bn\in \{0,\dots,N-1\}^d} a_{\Bn} \psi(N\Bx-\Bn): \Ba\in \cA_N \right\},
\end{equation}
where we denote $\cA_N:=\{ \Ba=(a_\Bn)_{\Bn \in\{0,\dots,N-1\}^d }: a_\Bn\in\{1,-1\}\}$ as the set of all sign vectors indexed by $\Bn$. Observe that, for the function $\psi_\Bn(\Bx) := \frac{C_{\psi,\alpha}}{N^{\alpha}} \psi(N\Bx-\Bn)$,
\begin{align*}
\sup_{\Bx\in \bR^d} |\partial^\Bs \psi_\Bn(\Bx)| &=  N^{\|\Bs\|_1-\alpha} C_{\psi,\alpha} \sup_{\Bx\in \bR^d} |\partial^\Bs \psi(\Bx)| \le 1, \quad &\|\Bs\|_1 \le r, \\
\sup_{\Bx\neq \By} \frac{|\partial^\Bs \psi_\Bn(\Bx)- \partial^\Bs \psi_\Bn(\By)|}{\|\Bx-\By\|_\infty^\beta} &= N^{r-\alpha} C_{\psi,\alpha} \sup_{\Bx\neq \By} \frac{|\partial^\Bs \psi(\Bx)- \partial^\Bs \psi(\By)|}{N^{-\beta} \|\Bx-\By\|_\infty^\beta} \le 1, \quad &\|\Bs\|_1 = r,
\end{align*}
where $\alpha = r+\beta>0$, with $r\in \bN_0, \beta\in (0,1]$ and we use the fact $C_{\psi,\alpha} \psi \in \cH^\alpha(\bR^d)$. Therefore, $\psi_\Bn$ is also in $\cH^\alpha(\bR^d)$. Since the functions $\psi_\Bn$ have disjoint supports and $a_{\Bn} \in \{1,-1\}$, one can check that each $h_\Ba$ is in $\cH^\alpha(\bR^d)$ and hence $\cH^\alpha_N \subseteq \cH^\alpha$.

Next, we consider the packing number of $\cH^\alpha_N$ on the set $\Lambda_N:=\{\Bn/N: \Bn\in \{0,\dots,N-1\}^d \}$. For convenience, we will denote the function values of a function class $\cF$ on $\Lambda_N$ by
\[
\cF(\Lambda_N) := \{ (f(\Bn/N))_{\Bn\in \{0,\dots,N-1\}^d} :f\in \cF \} \subseteq \bR^m,
\]
where $m=|\Lambda_N|=N^d$ is the cardinality of $\Lambda_N$. Observe that, for $h_\Ba \in \cH^\alpha_N$,
\begin{equation}\label{hs values}
h_\Ba(\Bn/N) = \frac{C_{\psi,\alpha}}{N^{\alpha}} \sum_{\Bi\in \{0,\dots,N-1\}^d} a_{\Bi} \psi(\Bn-\Bi) = \frac{C_{\psi,\alpha}}{N^{\alpha}} a_{\Bn},
\end{equation}
where the last equality is because $\psi(\Bn-\Bi)=1$ if $\Bn=\Bi$ and $\psi(\Bn-\Bi)=0$ if $\Bn\neq\Bi$. We conclude that
\[
\cH^\alpha_N(\Lambda_N) = \{ C_{\psi,\alpha} N^{-\alpha} \Ba: \Ba\in \cA_N \} = C_{\psi,\alpha} N^{-\alpha} \cA_N.
\]
We will estimate the packing number of $\cH^\alpha_N(\Lambda_N)$ under the metric
\begin{equation}\label{rho2}
\rho_2(\Bx,\By) := \left( \frac{1}{m} \sum_{i=1}^{m} (x_i-y_i)^2 \right)^{1/2} = m^{-1/2} \|\Bx-\By\|_2, \quad \Bx,\By\in \bR^m.
\end{equation}
The following combinatorial lemma is sufficient for our purpose.

\begin{lemma}\label{combination}
Let $\cA:=\{\Ba=(a_1,\dots,a_m): a_i\in\{1,-1\}\}$ be the set of all sign vectors on $\bR^m$.
For any $m\ge 8$, there exists a subset $\cB\subseteq \cA$ whose cardinality $|\cB|\ge 2^{m/4}$, such that any two sign vectors $\Ba\neq \Ba'$ in $\cB$ are different in more than $\lfloor m/8 \rfloor$ places.
\end{lemma}
\begin{proof}
For any $\Ba \in \cA$, let $U(\Ba)$ be the set of all $\Ba'$ which are different from $\Ba$ in at most $k=\lfloor m/8 \rfloor$ places. Then,
\[
|U(\Ba)| \le \sum_{i=0}^{k} \binom{m}{i} \le \left( \frac{me}{k} \right)^k \le (16e)^{m/8} \le 64^{m/8} = 2^{3m/4},
\]
where the second inequality is from \citet[Exercise 0.0.5]{vershynin2018high}. We can construct the set $\cB= \{\Ba_1,\dots,\Ba_n\}$ as follows. We take $\Ba_1\in \cA$ arbitrarily. Suppose the elements $\Ba_1,\dots,\Ba_j \in \cA$ have been chosen, then $\Ba_{j+1}$ is taken arbitrarily from $\cA\setminus (\cup_{i=1}^j U(\Ba_i))$. Then, by construction, $\Ba_{j+1}$ and $\Ba_i$ ($1\le i\le j$) are different in more than $\lfloor m/8 \rfloor$ places. We do this process until the set $\cA\setminus (\cup_{i=1}^n U(\Ba_i))$ is empty. Since
\[
2^m = |\cA| \le \sum_{i=1}^n |U(\Ba_i)|\le n 2^{3m/4},
\]
we must have $|\cB|=n\ge 2^{m/4}$.
\end{proof}

By Lemma \ref{combination}, when $m=N^d\ge 8$, there exists a subset $\cB_N\subseteq \cA_N$ whose cardinality $|\cB_N|\ge 2^{m/4}$, such that any two vectors $\Ba\neq \Ba'$ in $\cB_N$ are different in more than $\lfloor m/8 \rfloor$ places. Thus,
\[
\rho_2(\Ba,\Ba') = m^{-1/2} \|\Ba-\Ba'\|_2 \ge 2 m^{-1/2} \lfloor m/8 \rfloor^{1/2} > 1/2.
\]
By equation (\ref{hs values}), this implies that
\[
\rho_2(h_\Ba(\Lambda_N), h_{\Ba'}(\Lambda_N)) > \frac{C_{\psi,\alpha}}{2 N^{\alpha}}.
\]
In other words, $\{ h_\Ba(\Lambda_N): \Ba\in \cB_N \}$ is a $\frac{1}{2}C_{\psi,\alpha}N^{-\alpha}$-packing of $\cH^\alpha_N(\Lambda_N)$ and hence we can lower bound the packing number
\begin{equation}\label{pack lower bound}
\cN_p(\cH^\alpha(\Lambda_N), \rho_2, \tfrac{1}{2}C_{\psi,\alpha}N^{-\alpha}) \ge \cN_p(\cH^\alpha_N(\Lambda_N), \rho_2, \tfrac{1}{2}C_{\psi,\alpha}N^{-\alpha}) \ge 2^{m/4} = 2^{N^d/4}.
\end{equation}

On the other hand, one can upper bound the packing number of a set in $\bR^m$ by its Rademacher complexity due to Sudakov minoration for Rademacher processes, see \citet[Corollary 4.14]{ledoux1991probability} for example.

\begin{lemma}[Sudakov minoration]\label{Sudakov}
There exists a constant $C>0$ such that for any  set $S \subseteq \bR^m$ and any $\epsilon>0$,
\[
\log \cN_p(S,\rho_2,\epsilon) \le C \frac{m \cR_m(S)^2 \log \left(2+\frac{1}{\sqrt{m} \cR_m(S)} \right)}{\epsilon^2}.
\]
\end{lemma}

To simplify the notation, we denote $\Phi=\cN\cN(W,L,K)$. Lemma \ref{Rademacher bound} gives upper and lower bounds for the Rademacher complexity of $\Phi(\Lambda_N)$: for $K\ge 1$ and $W\ge 2$,
\[
\frac{1}{2\sqrt{2m}}\le \frac{K}{2\sqrt{2m}} \le \cR_m (\Phi(\Lambda_N)) \le \frac{K\sqrt{2(L+2+\log(d+1))}}{\sqrt{m}}.
\]
Together with Lemma \ref{Sudakov}, we can upper bound the packing number
\begin{equation}\label{pack upper bound}
\log \cN_p(\Phi(\Lambda_N), \rho_2, \epsilon) \le C \frac{K^2 L}{\epsilon^2},
\end{equation}
for some constant $C>0$.

Now, we are ready to prove our main lower bound for approximation error in Theorem \ref{app bounds}. The idea is that, if the approximation error $\cE(\cH^\alpha, \cN\cN(W,L,K))$ is small enough, then the packing numbers of $\cH^\alpha(\Lambda_N)$ and $\Phi(\Lambda_N)$ are close, and hence we can compare the lower bound (\ref{pack lower bound}) and upper bound (\ref{pack upper bound}). We will show that this leads to a contradiction when the approximation error is too small.

\begin{proof}[Proof of Theorem \ref{app bounds} (Lower bound)]
Denote $\Lambda_N:=\{\Bn/N: \Bn\in \{0,\dots,N-1\}^d \}$ and $\Phi=\cN\cN(W,L,K)$ as above. We have shown (by (\ref{pack lower bound}) and (\ref{pack upper bound})) that, when $N^d\ge 8$, there exists $C_1,C_2>0$ such that the packing number
\begin{equation}\label{pack lower bound2}
\log_2 \cN_p(\cH^\alpha(\Lambda_N), \rho_2, 3C_1N^{-\alpha}) \ge N^d/4,
\end{equation}
and for any $\epsilon>0$,
\begin{equation}\label{pack upper bound2}
\log_2 \cN_p(\Phi(\Lambda_N), \rho_2, \epsilon) \le C_2 \frac{K^2 L}{\epsilon^2}.
\end{equation}

Assume the approximation error $\cE(\cH^\alpha, \Phi) < C_1N^{-\alpha}$, where $N\ge 8^{1/d}$ will be chosen later. Using (\ref{pack lower bound2}), let $\cF$ be a subset of $\cH^\alpha$ such that $\cF(\Lambda_N)$ is a $3C_1N^{-\alpha}$-packing of $\cH^\alpha(\Lambda_N)$ with $\log_2 |\cF(\Lambda_N)|\ge N^d/4$. By assumption, for any $f_i\in \cF$, there exists $g_i\in \Phi$ such that $\|f_i-g_i\|_\infty \le C_1N^{-\alpha}$. Let $\cG$ be the collection of all $g_i$. Then, $\log_2|\cG(\Lambda_N)| \ge N^d/4$ and, for any $i\neq j$,
\begin{align*}
&\rho_2(g_i(\Lambda_N),g_j(\Lambda_N)) \\
\ge& \rho_2(f_i(\Lambda_N),f_j(\Lambda_N)) - \rho_2(f_i(\Lambda_N),g_i(\Lambda_N)) - \rho_2(g_j(\Lambda_N),f_j(\Lambda_N)) \\
\ge& \rho_2(f_i(\Lambda_N),f_j(\Lambda_N)) - \|f_i-g_i\|_\infty - \|g_j-f_j\|_\infty \\
>& 3C_1N^{-\alpha} - C_1N^{-\alpha} - C_1N^{-\alpha} \\
=& C_1N^{-\alpha}.
\end{align*}
In other words, $\cG(\Lambda_N)$ is a $C_1N^{-\alpha}$-packing of $\Phi(\Lambda_N)$. Combining with (\ref{pack upper bound2}), we have
\[
\frac{N^d}{4} \le \log_2 \cN_p(\Phi(\Lambda_N), \rho_2, C_1N^{-\alpha}) \le C_2 \frac{K^2 L}{C_1^2 N^{-2\alpha}},
\]
which is equivalent to
\begin{equation}\label{contradiction}
N^{d-2\alpha} \le 4C_1^{-2}C_2 K^2L.
\end{equation}
Now, we choose $N=\max\{ \lceil (5C_1^{-2}C_2 K^2L)^{1/(d-2\alpha)} \rceil, \lceil 8^{1/d} \rceil \}$, then (\ref{contradiction}) is always false. This contradiction implies $\cE(\cH^\alpha, \Phi) \ge C_1N^{-\alpha} \gtrsim (K^2L)^{-\alpha/(d-2\alpha)}$.
\end{proof}

Finally, we provide an alternative method to prove the lower bound in Theorem \ref{app bounds} when $\alpha=1$. We observe that, for any $f\in \cH^1$ and $\phi \in \cN\cN(W,L,K)$, by Hahn-Banach theorem,
\[
\| f- \phi\|_{C ([0,1]^d)} = \sup_{\|T\|\neq 0} \frac{|Tf- T\phi|}{\|T\|} \ge \sup_{\|T\|\neq 0} \frac{|Tf|- |T\phi|}{\|T\|},
\]
where $T$ is any bounded linear functional on $C ([0,1]^d)$ with operator norm $\|T\|\neq 0$. Thus, for any nonzero linear functional $T$,
\begin{align*}
&\cE(\cH^1, \cN\cN(W,L,K))
\ge \sup_{f\in \cH^1} \inf_{\phi \in \cN\cN(W,L,K)}  \frac{|Tf|- |T\phi|}{\|T\|} \\
\ge& \frac{1}{\|T\|} \left(\sup_{f\in \cH^1} |T f| - \sup_{\phi \in \cN\cN(W,L,K)} |T\phi|\right) = \frac{1}{\|T\|} \left(\sup_{f\in \cH^1} T f - \sup_{\phi \in \cN\cN(W,L,K)} T\phi\right).
\end{align*}
Hence, to provide a lower bound of $\cE(\cH^1, \cN\cN(W,L,K))$, we only need to find a linear functional $T$ that distinguishes $\cH^1$ and $\cN\cN(W,L,K)$. In order to use the Rademacher complexity bounds for neural networks (Lemma \ref{Rademacher bound}), we will consider the functional
\begin{equation}\label{Tn}
T_n h := \frac{1}{n} \sum_{i=1}^{n} h(\Bx_i) - \int_{[0,1]^d} h(\Bx) d\Bx, \quad h\in C([0,1]^d),
\end{equation}
where the points $\Bx_1,\dots,\Bx_n \in [0,1]^d$ will be chosen appropriately. Notice that, when $\{\Bx_i\}_{i=1}^n$ are randomly chosen from the uniform distribution on $[0,1]^d$, $T_nh$ is the difference of empirical average and expectation. The optimal transport theory \citep{villani2008optimal} provides a lower bound for $\sup_{f\in \cH^1}T_n f$, while the Rademacher complexity upper bounds $\sup_{\phi\in \cN\cN(W,L,K)}T_n \phi$ in expectation by symmetrization argument.

\begin{theorem}\label{explicit lower bound}
For any $W,L \in \bN$, $K \ge 1$ and $d\ge 3$,
\[
\cE(\cH^1, \cN\cN(W,L,K)) \ge c_d \left(K \sqrt{L+2+\log(d+1)}\right)^{-2/(d-2)},
\]
where $c_d = (d-2) 4^{-d/(d-2)} (d+1)^{-(d+1)/(d-2)}$.
\end{theorem}

\begin{proof}
Define the functional $T_n$ on $C([0,1]^d)$ by (\ref{Tn}). It is easy to check that $\|T_n\| \le 2$. We have shown that
\[
\cE(\cH^1, \cN\cN(W,L,K)) \ge \frac{1}{2} \left(\sup_{f\in \cH^1} T_n f - \sup_{\phi \in \Phi} T_n\phi\right)
\]
where we denote $\Phi=\cN\cN(W,L,K)$ to simplify the notation. Our analysis is divided into three steps.

\textbf{Step 1}: Lower bounding $\sup T_n f$. Observe that $\cH^1\subseteq \Lip 1$ and, for any $g\in \Lip 1$, the function $f= g-\min_{\Bx\in [0,1]^d} g(\Bx) \in \cH^1$ satisfies $T_n f = T_ng$. We conclude that
\[
\sup_{f\in \cH^1} T_n f = \sup_{g\in \Lip 1} T_n g.
\]
By the Kantorovich-Rubinstein duality \citep[Remark 6.5]{villani2008optimal},
\[
\sup_{g\in \Lip 1} T_n g = \cW_1 \left(\frac{1}{n} \sum_{i=1}^n \delta_{\Bx_i}, \cU \right) := \inf_{\mu} \int_{[0,1]^d \times [0,1]^d} \|\Bx-\By\|_\infty d\mu(\Bx,\By)
\]
is the $1$-Wasserstein distance between the discrete distribution $\frac{1}{n} \sum_{i=1}^n \delta_{\Bx_i}$ and the uniform distribution $\cU$ on $[0,1]^d$, where the infimum is taken over all joint probability distribution (also called coupling) $\mu$ on $[0,1]^d \times [0,1]^d$, whose marginal distributions are $\frac{1}{n} \sum_{i=1}^n \delta_{\Bx_i}$ and $\cU$ respectively. It is enough to estimate the $1$-Wasserstein distance.

We notice that, for any $r\in [0,1/2]$,
\begin{align*}
&\cU \left( \left\{ \By\in [0,1]^d: \min_{1\le i\le n} \|\Bx_i-\By\|_\infty \ge rn^{-1/d} \right\} \right) \\
=& 1- \cU \left( \left\{ \By\in [0,1]^d: \min_{1\le i\le n} \|\Bx_i-\By\|_\infty < rn^{-1/d} \right\} \right) \\
\ge & 1 - \sum_{i=1}^n \cU \left( \left\{ \By\in [0,1]^d: \|\Bx_i-\By\|_\infty < rn^{-1/d} \right\} \right) \\
\ge& 1 - n (2rn^{-1/d})^d = 1 - 2^dr^d.
\end{align*}
Hence, for any coupling $\mu$ and $r\in [0,1/2]$,
\begin{align*}
\int_{[0,1]^d \times [0,1]^d} \|\Bx-\By\|_\infty d\mu(\Bx,\By)
=& \int_{ \cup_{i=1}^n\{\Bx_i \} \times [0,1]^d} \|\Bx-\By\|_\infty d\mu(\Bx,\By) \\
\ge& \int_{ \cup_{i=1}^n\{\Bx_i \} \times [0,1]^d} \min_{1\le i\le n}\|\Bx_i-\By\|_\infty d\mu(\Bx,\By) \\
=& \int_{[0,1]^d} \min_{1\le i\le n}\|\Bx_i-\By\|_\infty d \cU(\By) \\
\ge& (1-2^dr^d) rn^{-1/d}.
\end{align*}
As a consequence, for any $n$ points $\Bx_1,\dots,\Bx_n \in [0,1]^d$,
\[
\sup_{f\in \cH^1} T_n f \ge \sup_{r\in [0,1/2]} (1-2^dr^d) rn^{-1/d} = 2^{-1}d(d+1)^{-1-1/d} n^{-1/d},
\]
where the supremum is attained when $r=2^{-1}(d+1)^{-1/d}$.

\textbf{Step 2}: Upper bounding $\sup T_n \phi$. Let $X_{1:n} = \{X_i\}_{i=1}^n$ be $n$ i.i.d. samples from the uniform distribution $\cU$ on $[0,1]^d$. We are going to upper bound
\begin{align*}
\cI_n := \bE_{X_{1:n}} \left[ \sup_{\phi\in \Phi} \left(\frac{1}{n} \sum_{i=1}^{n} \phi(X_i) - \int_{[0,1]^d} \phi(\Bx) d\Bx\right) \right]
= \bE_{X_{1:n}} \left[ \sup_{\phi\in \Phi} \left(\frac{1}{n} \sum_{i=1}^{n} \phi(X_i) - \bE_{X\sim \cU} [\phi(X)]\right) \right].
\end{align*}
We introduce a ghost sample dataset $X'_{1:n}=\{X_i'\}_{i=1}^n$ drawn i.i.d. from $\cU$, independent of  $X_{1:n}$. Then,
\begin{align*}
\cI_n =& \bE_{X_{1:n}} \left[\sup_{\phi\in \Phi} \left(\frac{1}{n} \sum_{i=1}^n \phi(X_i) - \bE_{X'_{1:n}}  \frac{1}{n} \sum_{i=1}^n \phi(X'_i)\right) \right] \\
\le& \bE_{X_{1:n},X'_{1:n}} \left[\sup_{\phi\in \Phi} \frac{1}{n} \sum_{i=1}^n (\phi(X_i) -  \phi(X'_i)) \right].
\end{align*}
Let $\xi_{1:n} = \{\xi_i\}_{i=1}^n$ be a sequence of i.i.d. Rademacher variables independent of $X_{1:n}$ and $X'_{1:n}$. Then, by symmetry, we can bound $\cI_n$ by Rademacher complexity:
\begin{align*}
\cI_n \le& \bE_{X_{1:n},X'_{1:n},\xi_{1:n}} \left[\sup_{\phi\in \Phi} \frac{1}{n} \sum_{i=1}^n \xi_i (\phi(X_i) -  \phi(X'_i)) \right] \\
\le& \bE_{X_{1:n},X'_{1:n},\xi_{1:n}} \left[\sup_{\phi\in \Phi} \frac{1}{n} \sum_{i=1}^n \xi_i \phi(X_i) + \sup_{\phi\in \Phi} \frac{1}{n} \sum_{i=1}^n -\xi_i \phi(X'_i) \right] \\
=& 2 \bE_{X_{1:n},\xi_{1:n}} \left[\sup_{\phi\in \Phi} \frac{1}{n} \sum_{i=1}^n \xi_i \phi(X_i) \right] \\
=& 2 \bE_{X_{1:n}} \left[ \cR_n(\Phi(X_{1:n})) \right],
\end{align*}
where we denote $\Phi(X_{1:n}) := \{(\phi(X_1),\dots,\phi(X_n))\in \bR^n: \phi \in \Phi\}$ and the second last equality is due to the fact that $X_i$ and $X'_i$ have the same distribution and the fact that $\xi_i$ and $-\xi_i$ have the same distribution.

By Lemma \ref{Rademacher bound}, for any $X_{1:n} \subseteq [0,1]^d$,
\[
\cR_n(\Phi(X_{1:n})) \le \sqrt{2}K \sqrt{L+2+\log(d+1)} n^{-1/2}.
\]
Hence, there exists $\Bx_1,\dots,\Bx_n \in [0,1]^d$ such that
\[
\sup_{\phi\in \cN\cN(W,L,K)} T_n \phi \le \cI_n \le 2\sqrt{2}K \sqrt{L+2+\log(d+1)} n^{-1/2}.
\]

\textbf{Step 3}: Optimizing $n$. We have shown that there exists $T_n$ such that
\begin{align*}
\cE(\cH^1, \cN\cN(W,L,K))
&\ge \frac{1}{2} \left(\sup_{f\in \cH^1} T_n f - \sup_{\phi \in \Phi} T_n\phi\right) \\
&\ge ds^{-1} n^{-1/d} - \sqrt{2}t n^{-1/2},
\end{align*}
where $s = 4(d+1)^{1+1/d}$ and $t = K \sqrt{L+2+\log(d+1)}$. In order to optimize over $n$, we can choose
\[
n = \left\lfloor (s t)^{\frac{2d}{d-2}} \right\rfloor.
\]
Then, since $s t\ge 2$, we have $n \ge (s t)^{\frac{2d}{d-2}} -1 \ge \frac{1}{2} (s t)^{\frac{2d}{d-2}}$ and
\begin{align*}
\cE(\cH^1, \cN\cN(W,L,K)) &\ge d s^{-1} (s t)^{-2/(d-2)} - 2t(s t)^{-d/(d-2)} \\
&= (d-2) s^{-d/(d-2)} t^{-2/(d-2)} \\
&= c_d \left(K \sqrt{L+2+\log(d+1)}\right)^{-2/(d-2)},
\end{align*}
where $c_d = (d-2) 4^{-d/(d-2)} (d+1)^{-(d+1)/(d-2)}$.
\end{proof}

\section{Applications to machine learning}\label{sec: application}

In this section, we apply Theorem \ref{app bounds} to two typical machine learning algorithms: regression by neural networks and distribution estimation by GANs. For regression, the goal is to estimate an unknown function $f_0$ from its noisy samples. One of the useful and effective methods is the empirical risk minimization, which estimates $f_0$ by minimizing some risk on the observed samples over some chosen hypothesis class. When $f_0$ is in some continuous function class and the hypothesis class is a ReLU neural network, the convergence rates of this estimator have been derived by \citep{schmidthieber2020nonparametric,nakada2020adaptive}. Here, we make a norm constraint on the weights and study the convergence rate of the corresponding estimator. As a consequence, our results provide statistical guarantee for  overparameterized networks, see Theorem \ref{convergence rate ERM} and Corollary \ref{convergence rate reg ERM}.  For distribution estimation, a GAN implicitly estimates the data distribution by training a generator that transports an easy-to-sample distribution to the data distribution, and a discriminator that distinguishes samples produced by the generator from true samples. It has been shown that GANs perform extremely well in practice \citep{gulrajani2017improved,miyato2018spectral,brock2019large}. We can combine the error analysis in \citet{huang2022error} with Theorem \ref{app bounds} to derive convergence rate for GANs with norm constrained neural networks as discriminator, which gives statistical guarantee on the performance of GANs, see Theorem \ref{GAN convergence rate} and Corollary \ref{GAN convergence rate reg}.

In the statistical analysis of learning algorithms, we often require that the hypothesis class is uniformly bounded. For any $B>0$, we will use the notations
\begin{align*}
\cN\cN_{d,k}^B(W,L) &:= \{\phi\in\cN\cN_{d,k}(W,L): \phi(\Bx)\in [-B,B]^k, \forall \Bx\in \bR^d\}, \\
\cN\cN_{d,k}^B(W,L,K) &:= \{\phi\in\cN\cN_{d,k}(W,L,K): \phi(\Bx)\in [-B,B]^k, \forall \Bx\in \bR^d\},
\end{align*}
which represent the neural network classes uniformly bounded by $B$. Note that we can truncate the output of $\phi\in \cN\cN_{d,k}(W,L,K)$ by applying $\chi_B(x) = (x \lor -B) \land B$ element-wise. Since
\[
\chi_B(x) = \sigma(x) - \sigma(-x) -(B+1)\sigma(\tfrac{x}{B+1}-\tfrac{B}{B+1}) +(B+1)\sigma(-\tfrac{x}{B+1}-\tfrac{B}{B+1}),
\]
it is not hard to see that $\chi_B\circ\phi \in \cN\cN_{d,k}^B(\max\{W,4k\},L+1,(2B+4)\max\{K,1\})$ by Proposition \ref{basic construct}. Therefore, the approximation upper bound in Theorem \ref{app bounds} also holds true for $\cN\cN^1(W,L,K)$ when $L \ge 2\lceil \log_2 (d+r) \rceil+3$.

\subsection{Regression}

Suppose we have a set of $n$ samples $S_n = \{(X_i,Y_i)\}_{i=1}^n \subseteq [0,1]^d \times \bR$ which are independently and identically generated from the regression model
\[
Y_i = f_0(X_i) + \eta_i, \quad X_i \sim \mu, \quad i=1,\dots,n,
\]
where $\mu$ is the marginal distribution of the covariates $X_i$ supported on $[0,1]^d$, and $\eta_i$ is an i.i.d. Gaussian noise independent of $X_i$ with $\bE[\eta_i]=0$ and $\bE[\eta_i^2] = V^2$, where $V\ge 0$. We aim to estimate the unknown target function $f_0\in \cH^\alpha$ by the empirical risk minimizer (ERM)
\begin{equation}\label{ERM}
\argmin_{\phi_\theta\in \cN\cN^1(W,L,K)} \cL_n(\phi_\theta) := \argmin_{\phi_\theta\in \cN\cN^1(W,L,K)} \frac{1}{n} \sum_{i=1}^n (\phi_\theta(X_i)- Y_i)^2.
\end{equation}
The performance of the estimation is measured by the expected risk
\[
\cL(\phi_\theta) := \bE_{(X,Y)} [(\phi_\theta(X)-Y)^2] = \bE_{X\sim \mu} [(\phi_\theta(X) - f_0(X))^2]  + V^2.
\]
It is equivalent to evaluate the estimator by the excess risk
\[
\| \phi_\theta - f_0\|^2_{L^2(\mu)} = \cL(\phi_\theta) - \cL(f_0).
\]

In deep learning, the optimization problem (\ref{ERM}) is generally solved by first order methods such as gradient descent or stochastic gradient descent  on the parameters $\theta$. Assume that $\widehat{\phi}_n$ is the output of a solver, say stochastic gradient descent, with optimization error $\epsilon_{opt}\ge 0$, i.e., %is a solution of the problem (\ref{ERM}) :
\begin{equation}\label{phi_n}
\cL_n(\widehat{\phi}_n) \le \inf_{\phi_\theta\in \cN\cN^1(W,L,K)} \cL_n(\phi_\theta) + \epsilon_{opt}.
\end{equation}
Then, for any $\phi_\theta\in \cN\cN^1(W,L,K)$,
\begin{align*}
&\| \widehat{\phi}_n - f_0\|^2_{L^2(\mu)} = \cL(\widehat{\phi}_n) - \cL(f_0) \\
=& \left[\cL(\widehat{\phi}_n) - \cL_n(\widehat{\phi}_n) \right] + \left[\cL_n(\widehat{\phi}_n) - \cL_n(\phi_\theta)\right] + \left[\cL_n(\phi_\theta) - \cL(\phi_\theta)\right] + \left[\cL(\phi_\theta) - \cL(f_0)\right] \\
\le& \left[\cL(\widehat{\phi}_n) - \cL_n(\widehat{\phi}_n)\right] + \epsilon_{opt} + \left[\cL_n(\phi_\theta) - \cL(\phi_\theta)\right]  + \| \phi_\theta - f_0\|^2_{L^2(\mu)}.
\end{align*}
Observing that $\bE_{S_n} \cL_n(\phi_\theta) = \cL(\phi_\theta)$ and taking the infimum over $\phi_\theta$, we get
\begin{equation}\label{err decomposition}
\bE_{S_n} \left[\| \widehat{\phi}_n - f_0\|^2_{L^2(\mu)} \right] \le \inf_{\phi \in \cN\cN^1(W,L,K)} \| \phi - f_0\|^2_{L^2(\mu)} + \bE_{S_n}\left[\cL(\widehat{\phi}_n) - \cL_n(\widehat{\phi}_n) \right] + \epsilon_{opt},
\end{equation}
where we decompose the excess risk into three terms: approximation error $\inf_{\phi} \| \phi - f_0\|^2_{L^2(\mu)}$, statistical (generalization) error $\bE_{S_n}[\cL(\widehat{\phi}_n) - \cL_n(\widehat{\phi}_n)]$ and optimization error $\epsilon_{opt}$.

\begin{theorem}\label{convergence rate ERM}
Assume $f_0\in \cH^\alpha$ with $\alpha = r+\beta>0$, where $r\in \bN_0 $ and $\beta\in (0,1]$. There exists $c>0$ such that for any $W\ge c K^{(2d+\alpha)/(2d+2)}$ and any $L \ge 2 \lceil \log_2 (d+r) \rceil+3$ independent of $n$, if we choose
\[
K\asymp n^{(d+1)/(2d+4\alpha+2)},
\]
then, for any estimator $\widehat{\phi}_n \in \cN\cN^1(W,L,K)$ satisfying (\ref{phi_n}),
\[
\bE_{S_n} \left[\| \widehat{\phi}_n - f_0\|^2_{L^2(\mu)} \right] -\epsilon_{opt} \lesssim n^{-\alpha/(d+2\alpha+1)} \log n.
\]
\end{theorem}
\begin{proof}
Using the error decomposition (\ref{err decomposition}), we only need to estimate the approximation error and stochastic error. For the approximation error, by Theorem \ref{app bounds} and the choice of $W$ and $L$,
\[
\inf_{\phi \in \cN\cN^1(W,L,K)} \| \phi - f_0\|^2_{L^2(\mu)} \lesssim K^{-2\alpha/(d+1)}.
\]
For the statistical error,
\begin{align*}
&\bE_{S_n}\left[\cL(\widehat{\phi}_n) - \cL_n(\widehat{\phi}_n) \right] \\
=& \bE_{S_n} \left[ \| \widehat{\phi}_n - f_0\|^2_{L^2(\mu)} + V^2 - \left(\frac{1}{n} \sum_{i=1}^n (\widehat{\phi}_n(X_i)- f_0(X_i))^2 -2 \eta_i(\widehat{\phi}_n(X_i)- f_0(X_i)) + \eta_i^2\right) \right] \\
=& \bE_{X_{1:n}}  \left[ \| \widehat{\phi}_n - f_0\|^2_{L^2(\mu)}  - \frac{1}{n} \sum_{i=1}^n (\widehat{\phi}_n(X_i)- f_0(X_i))^2 \right] +2 \bE_{S_n}\left[ \frac{1}{n} \sum_{i=1}^n \eta_i(\widehat{\phi}_n(X_i)- f_0(X_i)) \right] \\
\le& \bE_{X_{1:n}} \left[ \sup_{f\in \cF} \bE_X[f^2(X)] - \frac{1}{n} \sum_{i=1}^n f^2(X_i) \right] + 2 \bE_{X_{1:n}} \bE_{\eta_{1:n}} \left[ \sup_{f\in \cF} \frac{1}{n} \sum_{i=1}^n \eta_i f(X_i) \right]
\end{align*}
where $\cF:= \{ \phi-f_0: \phi\in \cN\cN^1(W,L,K) \}$ and $X_{1:n}=\{X_i\}_{i=1}^n$ is the sequence of samples. By a standard symmetrization argument (similar to step 2 in the proof of Theorem \ref{explicit lower bound}), one can obtain
\[
\bE_{X_{1:n}} \left[ \sup_{f\in \cF} \bE_X[f^2(X)] - \frac{1}{n} \sum_{i=1}^n f^2(X_i) \right] \le 2 \bE_{X_{1:n}} \left[ \cR_n(\cF^2(X_{1:n})) \right],
\]
where we denote $\cF^2(X_{1:n}):= \{ (f^2(X_1),\dots,f^2(X_n)) \in \bR^n: f\in \cF \} \subseteq \bR^n$. Since $\|f_0\|_\infty\le 1$ and $\|f\|_\infty = \|\phi-f_0\|_\infty \le 2$ for any $f\in \cF$, by the structural properties of Rademacher complexity (see \citet[Theorem 12]{bartlett2002rademacher}), we have
\[
\bE_{X_{1:n}} \left[ \cR_n(\cF^2(X_{1:n})) \right] \le 8  \bE_{X_{1:n}} \cR_n(\cF(X_{1:n})) \le 8\left( \bE_{X_{1:n}} \cR_n(\Phi(X_{1:n})) + \frac{\|f_0\|_\infty}{\sqrt{n}} \right) \lesssim \frac{K}{\sqrt{n}},
\]
where $\Phi(X_{1:n}):= \{ (\phi(X_1),\dots,\phi(X_n)) \in \bR^n: \phi\in \cN\cN^1(W,L,K) \}$ and we use Lemma \ref{Rademacher bound} in the last inequality. On the other hand, the Gaussian complexity can be bounded by Rademacher complexity \citep[Lemma 4]{bartlett2002rademacher}:
\[
\bE_{X_{1:n}} \bE_{\eta_{1:n}} \left[ \sup_{f\in \cF} \frac{1}{n} \sum_{i=1}^n \eta_i f(X_i) \right] \lesssim \bE_{X_{1:n}} \left[ \cR_n(\cF(X_{1:n})) \right] \log n \lesssim \frac{K \log n}{\sqrt{n}}.
\]
Hence,
\[
\bE_{S_n}\left[\cL(\widehat{\phi}_n) - \cL_n(\widehat{\phi}_n) \right] \lesssim \frac{K \log n}{\sqrt{n}}.
\]

In summary, the error decomposition (\ref{err decomposition}) implies
\[
\bE_{S_n} \left[\| \widehat{\phi}_n - f_0\|^2_{L^2(\mu)} \right] -\epsilon_{opt} \lesssim K^{-2\alpha/(d+1)} + \frac{K \log n}{\sqrt{n}}.
\]
If we choose $K \asymp n^{(d+1)/(2d+4\alpha+2)}$, then
\[
\bE_{S_n} \left[\| \widehat{\phi}_n - f_0\|^2_{L^2(\mu)} \right] -\epsilon_{opt} \lesssim n^{-\alpha/(d+2\alpha+1)} \log n. \qedhere
\]
\end{proof}

\begin{remark}
We have estimated the learning rate of the ERM in expectation (with respect to the observed samples). High probability bounds on the error $\| \widehat{\phi}_n - f_0\|^2_{L^2(\mu)}$ can be similarly derived by using concentration inequalities for random processes, see \citep{boucheron2013concentration,anthony2009neural,shalevshwartz2014understanding,mohri2018foundations} for more details.
\end{remark}

The constrained optimization problem (\ref{ERM}) may be difficult to optimize in practice. As an alternative, one can use the regularized empirical risk minimization
\begin{equation}\label{regularized ERM}
\argmin_{\phi_\theta\in \cN\cN^1(W,L)} \cL_{n,\lambda}(\phi_\theta) := \argmin_{\phi_\theta\in \cN\cN^1(W,L)} \frac{1}{n} \sum_{i=1}^n (\phi_\theta(X_i)- Y_i)^2 + \lambda\kappa(\theta), \quad \lambda> 0.
\end{equation}
Assume that $\widehat{\phi}_{n,\lambda} \in \cN\cN^1(W,L)$ parameterized by $\widehat{\theta}_{n,\lambda}$ is the output of an optimization solver, say stochastic gradient descent,   with optimization error $\epsilon_{opt}\ge 0$, i.e., $\widehat{\theta}_{n,\lambda}$ is an  $\epsilon_{opt}$-optimal solution of (\ref{regularized ERM}) satisfying 
\begin{equation}\label{phi_n,lamda}
\cL_{n,\lambda}(\widehat{\phi}_{n,\lambda}) \le \inf_{\phi_\theta\in \cN\cN^1(W,L)} \cL_{n,\lambda}(\phi_\theta) + \epsilon_{opt}.
\end{equation}
Then, for any $K\ge 0$ and $\phi_\theta \in \cN\cN^1(W,L,K)$, we have
\[
\cL_n(\widehat{\phi}_{n,\lambda}) + \lambda \kappa(\widehat{\theta}_{n,\lambda}) = \cL_{n,\lambda}(\widehat{\phi}_{n,\lambda}) \le \cL_n(\phi_{\theta}) + \lambda \kappa(\theta) + \epsilon_{opt}.
\]
Taking the infimum over all $\phi_\theta \in \cN\cN^1(W,L,K)$, we get
\begin{equation}\label{error for phi_n,lamda}
\cL_n(\widehat{\phi}_{n,\lambda}) + \lambda \kappa(\widehat{\theta}_{n,\lambda}) \le \inf_{\phi_\theta\in \cN\cN^1(W,L,K)} \cL_n(\phi_{\theta}) + \lambda K + \epsilon_{opt}.
\end{equation}
Hence, $\widehat{\phi}_{n,\lambda}$ can be regard as a solution of the constrained optimization problem (\ref{ERM}) with optimization error bounded by $\lambda K+\epsilon_{opt}$ for certain $K$. As a corollary, we show that the regularized ERM can achieve the same convergence rate of ERM in Theorem \ref{convergence rate ERM}, when there is no noise and $\lambda$ is chosen appropriately.

\begin{corollary}\label{convergence rate reg ERM}
Under the assumption of Theorem \ref{convergence rate ERM} with zero noise $\eta_i=Y_i-f_0(X_i)=0$, there exists $c>0$ such that for any
\[
W \ge c n^{(2d+\alpha)/(4d+8\alpha+4)},  \quad L \ge 2 \lceil \log_2 (d+r) \rceil+3, \quad \lambda \asymp n^{-1/2},
\]
and any estimator $\widehat{\phi}_{n,\lambda} \in \cN\cN^1(W,L)$ satisfying (\ref{phi_n,lamda}) with optimization error
\[
\epsilon_{opt} \lesssim n^{-\alpha/(d+2\alpha+1)},
\]
we have
\[
\bE_{S_n} \left[\| \widehat{\phi}_{n,\lambda} - f_0\|^2_{L^2(\mu)} \right] \lesssim n^{-\alpha/(d+2\alpha+1)} \log n.
\]
\end{corollary}
\begin{proof}
By Theorem \ref{app bounds}, there exists $c_0>0$ such that for any $W\ge c_0 K^{(2d+\alpha)/(2d+2)}$ and $L \ge 2 \lceil \log_2 (d+r) \rceil+3$,
\[
\cE(\cH^\alpha, \cN\cN(W,L,K)) \lesssim K^{-\alpha/(d+1)}.
\]
Since the noise $\eta_i=0$, inequality (\ref{error for phi_n,lamda}) implies
\begin{align*}
\kappa(\widehat{\theta}_{n,\lambda}) &\le \frac{1}{\lambda} \inf_{\phi_\theta\in \cN\cN^1(W,L,K)} \frac{1}{n} \sum_{i=1}^n (\phi_\theta(X_i)- f_0(X_i))^2 + K + \frac{\epsilon_{opt}}{\lambda} \\
& \lesssim \lambda^{-1}K^{-2\alpha/(d+1)} + K + \lambda^{-1} \epsilon_{opt}.
\end{align*}
If $\lambda \asymp K^{-1} K^{-2\alpha/(d+1)}$ and $\epsilon_{opt} \lesssim \lambda K$, then $\widehat{\phi}_{n,\lambda} \in \cN\cN^1(W,L,\widetilde{K})$ with $K \le \widetilde{K} \lesssim K$. Using inequality (\ref{error for phi_n,lamda}) again, we have
\[
\cL_n(\widehat{\phi}_{n,\lambda}) \le \inf_{\phi_\theta\in \cN\cN^1(W,L,\widetilde{K})} \cL_n(\phi_{\theta}) + \lambda \widetilde{K} + \epsilon_{opt},
\]
which implies $\widehat{\phi}_{n,\lambda}$ is a solution of the constrained optimization problem with optimization error $\lambda \widetilde{K}+\epsilon_{opt}$. Now, we choose $\widetilde{K} \asymp K\asymp n^{(d+1)/(2d+4\alpha+2)}$ and
\[
W \ge c_0 \widetilde{K}^{(2d+\alpha)/(2d+2)} \asymp n^{(2d+\alpha)/(4d+8\alpha+4)}.
\]
Then, $\lambda \asymp n^{-1/2}$ and $\epsilon_{opt} \lesssim \lambda K \lesssim n^{-\alpha/(d+2\alpha+1)}$. Therefore, Theorem \ref{convergence rate ERM} implies
\[
\bE_{S_n} \left[\| \widehat{\phi}_{n,\lambda} - f_0\|^2_{L^2(\mu)} \right] \lesssim n^{-\alpha/(d+2\alpha+1)} \log n + \epsilon_{opt} + \lambda \widetilde{K} \lesssim n^{-\alpha/(d+2\alpha+1)} \log n,
\]
which completes the proof.
\end{proof}

\begin{remark}
Thanks to the norm constraint, both Theorem \ref{convergence rate ERM} and Corollary \ref{convergence rate reg ERM} hold with no requirement on the upper bound of the size of network. As a consequence, we can allow the width $W$ and depth $L$ large enough such that the number of weights is greater than the number of samples, i.e.,  over-parameterization is allowed. 
Although the regularized optimization problem of the form (\ref{regularized ERM}) is highly nonconvex, for over-parameterized models, the optimization error $\epsilon_{opt}$ of stochastic gradient descent decays linearly to zero as the number of iterations increase under certain conditions \citep{allenzhu2019convergence, du2019gradient,nguyen2021proof,liu2022loss}.
Hence, with the help of the approximation results with norm constraint in this paper, it may be possible to close the gap between the current theory of approximation, generalization and optimization and further demystify why over-parameterized neural networks work well in practice.
\end{remark}

\subsection{Generative adversarial networks}

Suppose we have $n$ i.i.d. samples $S_n = \{X_i\}_{i=1}^n$ from an unknown probability distribution $\mu$ supported on $[0,1]^d$. Generative adversarial networks implicitly estimate the data distribution $\mu$ by training a generator $g:\bR^k\to [0,1]^d$ and a discriminator $f:[0,1]^d \to \bR$ against each other. To be concrete, we choose an easy-to-sample source distribution $\nu$ on $\bR^k$ (for example, uniform or Gaussian distribution) and compute the generator $g$ by minimizing the distance between the empirical distribution $\widehat{\mu}_n = \frac{1}{n}\sum_{i=1}^n \delta_{X_i}$ and the push-forward distribution $g_\# \nu$:
\begin{equation}\label{GAN}
\argmin_{g\in \cG} d_{\cF}(\widehat{\mu}_n, g_\# \nu) := \argmin_{g\in \cG} \sup_{f\in \cF} \bE_{\widehat{\mu}_n}[f] - \bE_{g_\# \nu} [f],
\end{equation}
where $d_\cF$ is the Integral Probability Metric (IPM, \citet{muller1997integral}) with respect to the discriminator class $\cF$, and the push-forward measure $g_\# \nu$ of a measurable set $S\subseteq [0,1]^d$ is defined by $g_\# \nu(S) = \nu(g^{-1}(S))$. In practice, the generator and discriminator classes are often parameterized by neural networks. If the training is successful, $g_\# \nu$ should be close to the target distribution $\mu$ in some sense. In general, we can evaluate the performance by another IPM with respect to the evaluation class $\cH$
\[
d_\cH(\mu,g_\# \nu):= \sup_{h\in \cH} \bE_{\mu}[h] - \bE_{g_\# \nu} [h].
\]
For instance, in the Wasserstein GAN \citep{arjovsky2017wasserstein}, $\cH =\Lip 1$ is the $1$-Lipschitz class and $d_\cH = \cW_1$ is the Wasserstein distance by Kantorovich-Rubinstein duality \citep{villani2008optimal}. In Sobolev GAN \citep{mroueh2018sobolev}, $\cH$ is a Sobolev class.

Assume that $\widehat{g}_n\in \cG$ is a solution of the problem (\ref{GAN}) with optimization error $\epsilon_{opt}\ge 0$:
\begin{equation}\label{g_n}
d_{\cF}(\widehat{\mu}_n, (\widehat{g}_n)_\# \nu) \le \argmin_{g\in \cG} d_{\cF}(\widehat{\mu}_n, g_\# \nu) + \epsilon_{opt}.
\end{equation}
Similar to the analysis for regression, we have the following error decomposition for GANs.

\begin{lemma}[\citet{huang2022error}, Lemma 9]\label{GAN error decomposition}
Assume that $\cF$ is symmetric ($f\in \cF$ implies $-f\in \cF$), $\mu$ and $g_\#\nu$ are supported on $[0,1]^d$ for all $g\in \cG$. Then, for any $\widehat{g}_n \in \cG$ satisfying (\ref{g_n}),
\[
d_{\cH^\alpha}(\mu,(\widehat{g}_n)_\# \nu) \le  2\cE(\cH^\alpha,\cF)  + \inf_{g \in \cG} d_\cF(\widehat{\mu}_n,g_\# \nu) + [d_{\cF}(\mu,\widehat{\mu}_n) \land d_{\cH^\alpha}(\mu,\widehat{\mu}_n)] + \epsilon_{opt}.
\]
\end{lemma}

Note that the error $d_{\cH^\alpha}(\mu,(\widehat{g}_n)_\# \nu)$ is decomposed into four error terms: (1) discriminator approximation error $\cE(\cH^\alpha,\cF)$ measuring how well the discriminator $\cF$ approximates the evaluation class $\cH^\alpha$; (2) generator approximation error $\inf_{g \in \cG} d_\cF(\widehat{\mu}_n,g_\# \nu)$ measuring the approximation capacity of the generator; (3) statistical error $d_{\cF}(\mu,\widehat{\mu}_n) \land d_{\cH^\alpha}(\mu,\widehat{\mu}_n)$ due to the fact that we only have finite samples; and (4) the optimization error $\epsilon_{opt}$. When $\cF = \cN\cN(W,L,K)$ is a class of norm constrained neural networks, Theorem \ref{app bounds} provides an upper bound on the discriminator approximation error. Since any function $f\in \cN\cN(W,L,K)$ is $K$-Lipschitz, the generator approximation error can be bounded by
\[
\inf_{g \in \cG} d_\cF(\widehat{\mu}_n,g_\# \nu) \le K \inf_{g \in \cG} \cW_1 (\widehat{\mu}_n,g_\# \nu),
\]
where $\cW_1 = d_{\Lip 1}$ is the Wasserstein distance. The approximation capacity of generative networks in Wasserstein distance have been studied recently by \citep{perekrestenko2020constructive,perekrestenko2021high,yang2022capacity}. Finally, the statistical error can be bounded using empirical process theory.

\begin{theorem}\label{GAN convergence rate}
Let $\mu$ be a probability distribution supported on $[0,1]^d$ and $\alpha = r+\beta>0$, where $r\in \bN_0 $ and $\beta\in (0,1]$. Assume that the generator $\cG$ and source distribution $\nu$ satisfy $\inf_{g \in \cG} \cW_1 (\widehat{\mu}_n,g_\# \nu)=0$ for any samples $S_n = \{X_i\}_{i=1}^n$. There exists $c>0$ such that, if the discriminator is chosen as $\cF=\cN\cN(W,L,K)$ with
\[
W\ge c K^{(2d+\alpha)/(2d+2)},\quad  L \ge 2 \lceil \log_2 (d+r) \rceil +2, \quad K \asymp n^{(d+1)/d},
\]
then, for any GAN estimator $\widehat{g}_n \in \cG$ satisfying (\ref{g_n}),
\[
\bE_{S_n}[d_{\cH^\alpha}(\mu,(\widehat{g}_n)_\# \nu)] - \epsilon_{opt} \lesssim n^{-\alpha/d} \lor n^{-1/2} (\log n)^\tau,
\]
where $\tau =1$ if $2\alpha =d$, and $\tau =0 $ otherwise.
\end{theorem}
\begin{proof}
By Theorem \ref{app bounds} and our choice of $W$ and $L$, the discriminator approximation error satisfies
\[
\cE(\cH^\alpha, \cF) \lesssim K^{-\alpha/(d+1)}.
\]
If we choose $K \asymp n^{(d+1)/d}$, then $\cE(\cH^\alpha, \cF) \lesssim n^{-\alpha/d}$. Since any $f\in \cF$ is $K$-Lipschitz,
\[
\inf_{g \in \cG} d_\cF(\widehat{\mu}_n,g_\# \nu) = \inf_{g \in \cG} \sup_{f\in \cF} \bE_{\widehat{\mu}_n}[f] - \bE_{g_\# \nu} [f] \le K \inf_{g \in \cG} \cW_1 (\widehat{\mu}_n,g_\# \nu) = 0,
\]
by assumption. Using a standard symmetrization argument (similar to step 2 in the proof of Theorem \ref{explicit lower bound}), the statistical error $\bE_{S_n}[d_{\cH^\alpha}(\mu,\widehat{\mu}_n)]$ can be bounded by Rademacher complexity, which can be further bounded by Dudley’s
entropy integral (see \citet[Lemma 12]{huang2022error} for more details):
\[
\bE_{S_n}[d_{\cH^\alpha}(\mu,\widehat{\mu}_n)] \lesssim \inf_{0<\delta<1/2} \left( \delta + \frac{1}{\sqrt{n}} \int_\delta^{1/2}  \sqrt{\log \cN_c(\cH^\alpha,\|\cdot\|_\infty,\epsilon)} d\epsilon \right).
\]
By \citet{kolmogorov1961}, we have the following bound for the covering number
\[
\log \cN_c(\cH^\alpha,\|\cdot\|_\infty,\epsilon) \lesssim \epsilon^{-d/\alpha}.
\]
Then, a simple calculation shows (see \citet[Page 9]{huang2022error})
\[
\bE_{S_n}[d_{\cH^\alpha}(\mu,\widehat{\mu}_n)] \lesssim n^{-\alpha/d} \lor n^{-1/2} (\log n)^\tau.
\]
The conclusion then follows from Lemma \ref{GAN error decomposition}.
\end{proof}

\begin{remark}
The assumption that the generator approximation error is zero can be fulfilled by sufficiently large neural network class $\cG=\cN\cN(W_1,L_1)$. More precisely, it was shown in \citep{yang2022capacity,huang2022error} that if $\nu$ is absolutely continuous and $n \lesssim W_1^2 L_1$ then $\inf_{g \in \cG} \cW_1 (\widehat{\mu}_n,g_\# \nu)=0$ for any samples $S_n = \{X_i\}_{i=1}^n$.
\end{remark}

\begin{remark}
For nonparametric density estimation, \citet{liang2021how,singh2018nonparametric} established the minimax optimal rate $\cO(n^{-(\alpha+\beta)/(2\beta+d)} \lor n^{-1/2})$ for learning distributions in a Sobolev class with smoothness $\beta$, when the evaluation class is another Sobolev class with smoothness $\alpha$. The learning rate in Theorem \ref{GAN convergence rate} matches this optimal rate with $\beta=0$ up to a logarithmic factor, without making any assumptions on the regularity of the target distribution.

\end{remark}

\begin{remark}
The optimization problem (\ref{GAN}) implicitly assume that we can compute the expectation $\bE_{g_\# \nu} [f] = \bE_{\nu} [f \circ g]$. This expectation can be estimated by the empirical average $\bE_{\widehat{\nu}_m} [f \circ g]$, where $\widehat{\nu}_m = \frac{1}{m}\sum_{i=1}^m \delta_{Z_i}$ is the empirical distribution of $m$ random samples $\{Z_i\}_{i=1}^m$ from $\nu$. Since $\nu$ is easy to sample, we can take $m$ as large as we want. Hence, in stead of (\ref{GAN}), one can use
\[
\argmin_{g\in \cG} d_{\cF}(\widehat{\mu}_n, g_\# \widehat{\nu}_m) := \argmin_{g\in \cG} \sup_{f\in \cF} \bE_{\widehat{\mu}_n}[f] - \bE_{g_\# \widehat{\nu}_m} [f].
\]
Suppose $\widehat{g}_{n,m}\in \cG$ is a solution with optimization error $\epsilon_{opt}$. Using the argument in \citet{huang2022error}, one can show that $\widehat{g}_{n,m}$ achieves the same rate as $\widehat{g}_n$ in Theorem \ref{GAN convergence rate}, if $m$ is sufficiently large.
\end{remark}

It has been demonstrated that Lipschitz continuity of the discriminator is a key condition for a stable training of GANs \citep{arjovsky2017towards,arjovsky2017wasserstein}. In the original Wasserstein GAN \citep{arjovsky2017wasserstein}, the Lipschitz constraint on the discriminator is implemented by weight clipping. In the follow-up works, several regularization methods have been proposed to enforce Lipschitz condition, such as gradient penalty \citep{gulrajani2017improved,petzka2018regularization}, weight normalization \citep{miyato2018spectral} and weight penalty \citep{brock2019large}. In Theorem \ref{GAN convergence rate}, the Lipschitz constant is controlled by the norm constraint $\kappa(\theta) \le K$. We can also estimate the convergence rate of the corresponding GAN estimator regularized by weight penalty:
\begin{equation}\label{regularized GAN}
\argmin_{g\in \cG} d_{\cF,\lambda}(\widehat{\mu}_n, g_\# \nu) := \argmin_{g\in \cG} \sup_{\phi_\theta\in \cF} \bE_{\widehat{\mu}_n}[\phi_\theta] - \bE_{g_\# \nu} [\phi_\theta] - \lambda \kappa(\theta)^2, \quad \lambda> 0,
\end{equation}
where $\cF=\cN\cN(W,L)$ is a neural network class. The following proposition explains the relation between the regularized problem (\ref{regularized GAN}) and the constrained optimization problem (\ref{GAN}).

\begin{proposition}\label{regularized IPM}
For any probability distributions $\mu$ and $\nu$ defined on $\bR^d$, any $\lambda,K>0$,
\[
d_{\cF,\lambda}(\mu,\nu) = \frac{d_{\cF_K}(\mu,\nu)^2}{4\lambda K^2},
\]
where $\cF=\cN\cN(W,L)$ and $\cF_K:=\cN\cN(W,L,K)$.
\end{proposition}
\begin{proof}
Observe that, for any $a\ge 0$,
\[
\sup_{\phi_\theta\in \cF, \kappa(\theta)=a} \bE_\mu[\phi_\theta] - \bE_\nu[\phi_\theta] = a \sup_{\phi_\theta\in \cF, \kappa(\theta)=1} \bE_\mu[\phi_\theta] - \bE_\nu[\phi_\theta],
\]
because if $\phi_\theta$ is parameterized by $\theta= ((A_0,\Bb_0),\dots,(A_{L-1},\Bb_{L-1}), A_L)$, then $a\phi_\theta$ can be parameterized by $\theta'= ((A_0,\Bb_0),\dots,(A_{L-1},\Bb_{L-1}), a A_L)$ and $\kappa(\theta')= a \kappa(\theta)$. Thus,
\[
d_{\cF_K}(\mu,\nu) = \sup_{0\le a\le K} \sup_{\phi_\theta\in \cF, \kappa(\theta)=a} \bE_\mu[\phi_\theta] - \bE_\nu[\phi_\theta] = K \sup_{\phi_\theta\in \cF, \kappa(\theta)=1} \bE_\mu[\phi_\theta] - \bE_\nu[\phi_\theta].
\]
Therefore,
\begin{align*}
d_{\cF,\lambda}(\mu,\nu) &= \sup_{\phi_\theta\in \cF} \bE_{\mu}[\phi_\theta] - \bE_{\nu} [\phi_\theta] - \lambda \kappa(\theta)^2 \\
&= \sup_{a\ge 0} \sup_{\phi_\theta\in \cF, \kappa(\theta)=a} \bE_\mu[\phi_\theta] - \bE_\nu[\phi_\theta] - \lambda a^2 \\
&= \sup_{a\ge 0} \frac{a}{K} d_{\cF_K}(\mu,\nu) - \lambda a^2 \\
&= \frac{d_{\cF_K}(\mu,\nu)^2}{4\lambda K^2},
\end{align*}
where the supremum is achieved at $a=\frac{1}{2\lambda K} d_{\cF_K}(\mu,\nu)$ in the last equality.
\end{proof}

Combining Proposition \ref{regularized IPM} with Theorem \ref{GAN convergence rate}, we can obtain the learning rate of the solution of the regularized optimization problem (\ref{regularized GAN}).

\begin{corollary}\label{GAN convergence rate reg}
Under the assumption of Theorem \ref{GAN convergence rate}, let $W,L,K$ be the parameters in Theorem \ref{GAN convergence rate} and $\lambda = \frac{1}{4K^2} \asymp n^{-2(d+1)/d}$, then for any GAN estimator $\widehat{g}_{n,\lambda} \in \cG$ satisfying
\[
d_{\cF,\lambda}(\widehat{\mu}_n, (\widehat{g}_{n,\lambda})_\# \nu) \le \argmin_{g\in \cG} d_{\cF,\lambda}(\widehat{\mu}_n, g_\# \nu) + \epsilon_{opt},
\]
where $\cF=\cN\cN(W,L)$, we have
\[
\bE_{S_n}[d_{\cH^\alpha}(\mu,(\widehat{g}_{n,\lambda})_\# \nu)] - \sqrt{\epsilon_{opt}} \lesssim n^{-\alpha/d} \lor n^{-1/2} (\log n)^\tau,
\]
where $\tau =1$ if $2\alpha =d$, and $\tau =0 $ otherwise.
\end{corollary}
\begin{proof}
Since $\lambda= \frac{1}{4K^2}$, by Proposition \ref{regularized IPM},
\begin{align*}
d_{\cF_K}(\widehat{\mu}_n, (\widehat{g}_{n,\lambda})_\# \nu)^2 &= d_{\cF,\lambda}(\widehat{\mu}_n, (\widehat{g}_{n,\lambda})_\# \nu) \le \argmin_{g\in \cG} d_{\cF,\lambda}(\widehat{\mu}_n, g_\# \nu) + \epsilon_{opt} \\
&= \argmin_{g\in \cG} d_{\cF_K}(\widehat{\mu}_n, g_\# \nu)^2 + \epsilon_{opt},
\end{align*}
where we denote $\cF_K=\cN\cN(W,L,K)$. As a consequence,
\[
d_{\cF_K}(\widehat{\mu}_n, (\widehat{g}_{n,\lambda})_\# \nu) \le \sqrt{\argmin_{g\in \cG} d_{\cF_K}(\widehat{\mu}_n, g_\# \nu)^2 + \epsilon_{opt}} \le \argmin_{g\in \cG} d_{\cF_K}(\widehat{\mu}_n, g_\# \nu) + \sqrt{\epsilon_{opt}},
\]
which means $\widehat{g}_{n,\lambda} \in \cG$ is a solution of (\ref{GAN}) with discriminator $\cF_K$ and optimization error $\sqrt{\epsilon_{opt}}$. Hence, we can apply Theorem \ref{GAN convergence rate}.
\end{proof}

\section{Conclusions and future work}\label{sec: conclusion}

This paper has established upper and lower approximation bounds for ReLU neural networks with norm constraint on the weights. We used these bounds to analyze the convergence rate of estimating H\"older continuous functions by norm constrained neural networks. In particular, our results can be applied to over-parameterized neural networks, which are widely used in practice. We also showed that GAN can achieve optimal rate of learning probability distributions, when the discriminator is a properly chosen norm constrained neural network. Our results provide statistical guarantees on the performance of norm constrained neural networks.

Norm constrained or regularized neural networks have been widely used in practical applications \citep{neyshabur2015path,miyato2018spectral,brock2019large}. But the theory of their approximation and generalization capacity is still very limited. We hope that this work can motivate more study on this field. In the following, we list some possible directions for future research.
\begin{itemize}[parsep=0pt]
\item There is a gap between the upper and lower bounds in Theorem \ref{app bounds}. In \citep{yarotsky2018optimal,shen2020deep}, the optimal approximation rates, in terms of the numbers of weights and neurons, are derived through the so-called bit extraction technique \citep{bartlett2019nearly}. By using this technique, one can approximately discretize the input and reduce the approximation problem to an interpolation problem \citep{shen2020deep,lu2021deep}. This helps us avoid computing the outer summation in the local Taylor approximation (\ref{taylor}). Hence, we think it is worth to explore whether one can apply bit extraction technique to construct norm constrained neural networks that have better approximation rates.

\item The lower bound in Theorem \ref{app bounds} is derived through the upper bound for Rademacher complexity in Lemma \ref{Rademacher bound}. This upper bound is independent of the width, but depends on the depth. It is still unclear whether it is possible to obtain size-independent bounds without further assumption on the weights of neural networks.

\item In the definition of norm constraint (\ref{norm constraint}), we restrict ourselves to the operator norm induced by $\|\cdot\|_\infty$ for the weight matrices. It will be interesting to extend the results to other norms. A more fundamental question is how different norms affect the approximation and generalization capacity?
\end{itemize}

\section*{Acknowledgments}

The work of Y. Jiao is supported in part by the National Natural Science Foundation of China under Grant 11871474 and by the research
fund of KLATASDSMOE. The research of Y. Wang is supported by the HK RGC grant 16308518, the HK Innovation Technology Fund Grant  ITS/044/18FX and the Guangdong-Hong Kong-Macao Joint Laboratory for Data Driven Fluid Dynamics and Engineering Applications (Project 2020B1212030001). Y. Yang is grateful for the support from Huawei. We thank the anonymous reviewers for their helpful comments and suggestions.

\bibliographystyle{myplainnat}
\bibliography{References}
\end{document}